\documentclass{article}

     \PassOptionsToPackage{numbers, compress}{natbib}

\usepackage[final]{neurips_2020}

\usepackage[utf8]{inputenc} 
\usepackage[T1]{fontenc}    
\usepackage{url}            
\usepackage{booktabs}       
\usepackage{amsfonts}       
\usepackage{nicefrac}       
\usepackage{microtype}      
\usepackage{xcolor}
\definecolor{orange}{rgb}{0.0,0.0,0.75}
\definecolor{darkgreen}{rgb}{0.015,0.6215,0.4413}
\usepackage{tikz} 
\usetikzlibrary{shapes,arrows.meta} 
\usetikzlibrary{positioning}
\usepackage[colorlinks=true]{hyperref}
\hypersetup{
	pdfborder={0 0 0.4},      
	urlbordercolor=red,
	citecolor=orange,
}
\usepackage{textcomp}
\usepackage{amsmath, amsthm, mathtools}
\mathtoolsset{showonlyrefs}
\DeclareMathOperator*{\argmax}{arg\,max}
\DeclareMathOperator*{\argmin}{arg\,min}
\newcommand{\norm}[1]{\left\lVert#1\right\rVert}
\usepackage{multirow}
\usepackage{algorithm} 
\usepackage{algpseudocode, array, multirow} 
\newtheorem{theorem}{Theorem}[section]
\newtheorem{corollary}{Corollary}[theorem]
\newtheorem{lemma}[theorem]{Lemma}
\usepackage{graphicx}
\usepackage{subcaption}
\graphicspath{{./Figures/}}
\usepackage{textcomp}
\renewcommand{\textuparrow}{$\uparrow$}
\renewcommand{\textdownarrow}{$\downarrow$}
\usepackage{rotating}
\definecolor{edit}{rgb}{0., 0., 0.}
\definecolor{edit2}{rgb}{0., 0., 0.}

\title{AdvFlow: Inconspicuous Black-box Adversarial Attacks using Normalizing Flows}

\author{%
  Hadi M.~Dolatabadi, Sarah Erfani, Christopher Leckie\\
  School of Computing and Information Systems\\
  The University of Melbourne\\
  Parkville, Victoria, Australia \\
  \texttt{hadi.mohagheghdolatabadi@student.unimelb.edu.au} \\
}

\begin{document}
	
\maketitle

\begin{abstract}
    Deep learning classifiers are susceptible to well-crafted, imperceptible variations of their inputs, known as adversarial attacks.
	In this regard, the study of powerful attack models sheds light on the sources of vulnerability in these classifiers, hopefully leading to more robust ones.
	In this paper, we introduce AdvFlow: a novel black-box adversarial attack method on image classifiers that exploits the power of normalizing flows to model the density of adversarial examples around a given target image.
	We see that the proposed method generates adversaries that closely follow the clean data distribution, a property which makes their detection less likely.
	Also, our experimental results show competitive performance of the proposed approach with some of the existing attack methods on defended classifiers.
	The code is available at \url{https://github.com/hmdolatabadi/AdvFlow}.
\end{abstract}

\section{Introduction}\label{sec:introduction}

Deep neural networks~(DNN) have been successfully applied to a wide variety of machine learning tasks.
For instance, trained neural networks can reach human-level accuracy in image classification~\citep{russakovsky2015imagenet}.
However, \citet{szegedy2014intriguing} showed that such classifiers can be fooled by adding an imperceptible perturbation to the input image.
Since then, there has been \textcolor{edit}{extensive} research in this area known as \textit{adversarial machine learning}, trying to design more powerful attacks and devising more robust neural networks.
Today, this area encompasses a broader type of data than images, with video~\citep{jiang2019black}, graphs~\citep{zugner2018adversarial}, text~\citep{liang2018deep}, and other types of data classifiers being attacked.

In this regard, the design of stronger adversarial attacks plays a crucial role in understanding the nature of possible real-world threats.
The ultimate goal of such studies is to help neural networks \textcolor{edit}{become more robust} against such adversaries.
This line of research is extremely important as even the slightest flaw in some real-world applications of DNNs such as self-driving cars can \textcolor{edit}{have} severe, irreparable consequences~\citep{eykholt2018robust}.

In general, adversarial attack approaches \textcolor{edit}{can be} classified into two broad categories: white-box and black-box.
In \textit{white-box} adversarial attacks, the assumption is that the threat model has full access to the target DNN.
This way, adversaries can leverage their knowledge about the target model to generate adversarial examples (for instance, by taking the gradient of the neural network).
In contrast, \textit{black-box} attacks assume that they do not know the internal structure of the target model a priori.
Instead, they can only \textit{query} the model about some inputs, and work with the labels or confidence levels associated with them~\citep{yuan2019adversarial}.
Thus, black-box attacks seem to be making more realistic assumptions. 
In the beginning, black-box attacks were mostly thought of as the transferability of white-box adversarial examples to unseen models~\citep{papernot2016transfer}.
Recently, however, there has been more research to attack black-box models directly.

In this paper, we introduce AdvFlow: a black-box adversarial attack that makes use of pre-trained normalizing flows to generate adversarial examples.
In particular, we utilize flow-based methods pre-trained on clean data to model the probability distribution of possible adversarial examples around a given image.
Then, by exploiting the notion of \textit{search gradients} from \textit{natural evolution strategies~(NES)}~\citep{wierstra2008nes, wierstra2014natural}, we solve the black-box optimization problem associated with adversarial example generation to adjust this distribution.
At the end of this process, we wind up having a data distribution whose realizations are likely to be adversarial.
Since this density is constructed on the top of the original data distribution estimated by normalizing flows, we see that the generated perturbations take \textcolor{edit}{on} the structure of data rather than an additive noise (see Figure~\ref{fig:celeba_example}).
This property impedes distinguishing AdvFlow examples from clean data for adversarial example detectors, as they often assume that the adversaries come from a different distribution \textcolor{edit}{than the} clean data.
Moreover, we prove a lemma to conclude that adversarial perturbations generated by the proposed approach can be approximated by a normal distribution with dependent components.
We then put our model under test and show its effectiveness in generating adversarial examples with 1) less detectability, 2) higher success rate, 3) lower number of queries, and 4) higher rate of transferability on defended models compared to the similar method of $\mathcal{N}$\textsc{Attack}~\citep{li2019nattack}.

In summary, we make the following contributions:
\begin{itemize}
    \item We introduce AdvFlow, a black-box adversarial attack that leverages the power of normalizing flows in modeling data distributions. To the best of our knowledge, this is the first work that explores the use of flow-based models in the design of adversarial attacks.
	\item We prove a lemma about the adversarial perturbations generated by AdvFlow. As a result of this lemma, we deduce that AdvFlows can generate perturbations with dependent elements, while this is not the case for $\mathcal{N}$\textsc{Attack}~\citep{li2019nattack}.
	\item We show the power of the proposed approach in generating adversarial examples that have a similar distribution to the data. As a result, our method is able to mislead adversarial example detectors for they often assume adversaries come from a different distribution than the clean data. We then see the performance of the proposed approach in attacking some of the most recent adversarial training defense techniques. 
\end{itemize}

\begin{figure}[tb!]
	\centering
	\begin{subfigure}{.16\textwidth}
		\centering
		\includegraphics[width=1.\textwidth]{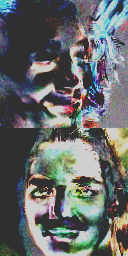}
		\caption*{(a)}
	\end{subfigure}
	\hspace*{-0.46em}
	\begin{subfigure}{.16\textwidth}
		\centering
		\includegraphics[width=1.\textwidth]{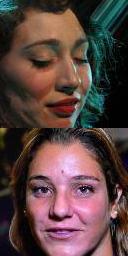}
		\caption*{(b)}
	\end{subfigure}
	\hspace*{-0.46em}
	\begin{subfigure}{.16\textwidth}
		\centering
		\includegraphics[width=1.\textwidth]{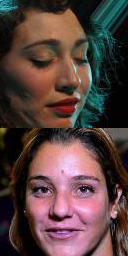}
		\caption*{(c)}
	\end{subfigure}
	\hspace*{-0.46em}
	\begin{subfigure}{.16\textwidth}
		\centering
		\includegraphics[width=1.\textwidth]{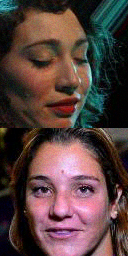}
		\caption*{(d)}
	\end{subfigure}
	\hspace*{-0.46em}
	\begin{subfigure}{.16\textwidth}
		\centering
		\includegraphics[width=1.\textwidth]{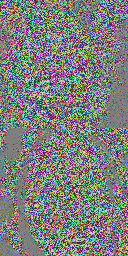}
		\caption*{(e)}
	\end{subfigure}
	\caption{Adversarial perturbations generated by AdvFlow take the structure of the original image into account, resulting in less detectable adversaries compared to $\mathcal{N}$\textsc{Attack}~\citep{li2019nattack} (see Section~\ref{sec:sec:detectability}).
	The classifier is a VGG19~\citep{simonyan2015vgg} trained to detect smiles in CelebA~\citep{liu2015deep} faces.
    (a) AdvFlow magnified difference (b) AdvFlow adversarial example (c) clean image (d) $\mathcal{N}$\textsc{Attack} adversarial example (e) $\mathcal{N}$\textsc{Attack} magnified difference.}
	\label{fig:celeba_example}
\end{figure}

\section{Related Work}\label{sec:related_works}

In this section, we review some of the most \textcolor{edit}{closely} related work to our proposed approach.
For a complete review of (black-box) adversarial attacks, we refer the interested reader to~\citep{yuan2019adversarial, bhambri2019study}.

\paragraph{Black-box Adversarial Attacks.}
In one of the earliest black-box approaches, \citet{chen2017zoo} used the idea of Zeroth Order Optimization and came up with a method called \textit{ZOO}.
In particular, ZOO uses the target neural network queries to build up a zero-order gradient estimator.
Then, it utilizes the estimated gradient to minimize a \textcolor{edit}{Carlini and Wagner} (C\&W) loss~\citep{carlini2017towards} and find an adversarial image.
Later and inspired by~\citep{wierstra2014natural, salimans2017evolution}, \citet{ilyas2018black} tried to estimate the DNN gradient using a normally distributed search density.
In particular, they estimate the gradient of the classifier ${\mathcal{C}(\mathbf{x})}$ with
\begin{equation}\label{eq:QL}\nonumber
\nabla_{\mathbf{x}}\mathcal{C}(\mathbf{x})\approx\mathbb{E}_{\mathcal{N}(\mathbf{z}|\mathbf{x}, \sigma^2 I)}\left[\mathcal{C}(\mathbf{z}) \nabla_{\mathbf{x}}\log\big(\mathcal{N}(\mathbf{z}|\mathbf{x}, \sigma^2 I)\big)\right],
\end{equation}
which only requires querying the black-box model ${\mathcal{C}(\mathbf{x})}$.
Having the DNN gradient estimate, \citet{ilyas2018black} then take a \textit{projected gradient descent (PGD)} step to minimize their objective for generating an adversarial example.
This idea is further developed in the construction of \textit{$\mathcal{N}$\textsc{Attack}}~\citep{li2019nattack}.
Specifically, instead of trying to minimize the adversarial example generation objective directly, they aim to fit a distribution around the clean data so that its realizations are likely to be adversarial (\textcolor{edit}{see} Section~\ref{sec:sec:NES_NATTACK} for more details).
In another piece of work, \citet{ilyas2019prior} observe that the gradients used in adversarial example generation by PGD exhibit a high correlation both in time and across data.
Thus, the number of queries to attack a black-box model can be reduced if one incorporates this prior knowledge about the gradients.
To this end, \citet{ilyas2019prior} uses a bandit-optimization technique to integrate these priors into their attack, resulting in a method called \textit{Bandits \& Priors}.
Finally, \textit{Simple Black-box Attack}~(SimBA)~\citep{guo2019simba} is a straightforward, intuitive approach to construct black-box adversarial examples.
It is first argued that for any particular direction $\mathbf{q}$ and step size~$\epsilon>0$, either $\mathbf{x} - \epsilon \mathbf{q}$ or $\mathbf{x} + \epsilon \mathbf{q}$ is going to decrease the probability of detecting the correct class label of the input image $\mathbf{x}$.
Thus, we are likely to find an adversary by iteratively taking such steps.
The vectors $\mathbf{q}$ are selected from a set of orthonormal candidate vectors $Q$.
\citet{guo2019simba} use Discrete Cosine Transform~(DCT) to construct such a set, exploiting the observation that ``random noise in low-frequency space is more likely to be adversarial''~\citep{guo2019lowfreq}.

\paragraph{Adversarial Attacks using Generative Models.}

There has been some prior work that utilizes the power of generative models (mostly generative adversarial networks~(GAN)) to model adversarial perturbations and attack DNNs~\citep{baluja2018learning, xiao2018generating, wang2019adirect, huang2019black}.
The target of these models is mainly white-box attacks.
They require training of their parameters to produce adversarial perturbations using a cost function that involves taking the gradient of a target network.
To adapt themselves to black-box settings, they try to replace this target network with either a distilled version of it~\citep{xiao2018generating}, or a substitute source model~\cite{huang2019black}.
However, as we will see in Section~\ref{sec:proposed_method}, the flow-based part of our model is only pre-trained on some clean training data using the maximum likelihood objective of Eq.~\eqref{eq:ML}.
Thus, AdvFlow can be adapted to any target classifier of the same dataset, without the need to train it again.
Moreover, while prior work is mainly concerned with generating the adversarial perturbations (for example~\citep{tu2019autozoom}), here we use the normalizing flows output as the adversarial example directly.
In this sense, our work is more similar to~\cite{song2018constructing} that generates unrestricted adversarial examples using GANs in a white-box setting, and falls under functional adversarial attacks~\citep{laidlaw2019functional}.
However, besides being black-box, in AdvFlow we restrict the output to be in the vicinity of the original image.

\section{Proposed Method}\label{sec:proposed_method}

In this section, we propose our attack method.
First, we define the problem of black-box adversarial attacks formally.
\textcolor{edit}{Next, we go over normalizing flows and see how we can train a flow-based model.}
Then, we review the idea of Natural Evolution Strategies~(NES)~\citep{wierstra2008nes, wierstra2014natural} and $\mathcal{N}$\textsc{Attack}~\citep{li2019nattack}.
Afterward, we show how normalizing flows can be mixed with NES in the context of black-box adversarial attacks, resulting in a method we call AdvFlow.
Finally, we prove a lemma about the nature of the perturbations generated by the proposed approach and show that $\mathcal{N}$\textsc{Attack} cannot produce the adversarial perturbations generated by AdvFlow.
Our results in Section \ref{sec:experimental_results} support this lemma.
There, we see that AdvFlow can generate adversarial examples that are less detectable than the ones generated by $\mathcal{N}$\textsc{Attack}~\citep{li2019nattack} due to its perturbation structure.

\subsection{Problem Statement}\label{sec:sec:problem_statement}

Let ${\mathcal{C}(\cdot): \mathcal{X}^{d} \rightarrow \mathcal{P}^{k}}$ denote a DNN classifier.
Assume that the classifier \textcolor{edit}{takes} a ${d}$-dimensional \textcolor{edit}{input} ${\mathbf{x} \in \mathcal{X}^{d}}$, and outputs a vector ${\mathbf{p} \in \mathcal{P}^{k}}$.
Each element of the vector ${\mathbf{p}}$ indicates the probability of the input belonging to one of the ${k}$ classes that the classifier is trying to distinguish.
Furthermore, \textcolor{edit}{let} ${y}$ denote the correct class label of the data.
In other words, if the ${y}$-th element of the classifier output ${\mathbf{p}}$ is larger than the rest, then the input has been correctly classified.
Finally, let the well-known Carlini and Wagner (C\&W) loss~\citep{carlini2017towards} be defined as\footnote{
	Note that although we are defining our objective function $\mathcal{L}(\mathbf{x}')$ for un-targeted adversarial attacks, it can be easily modified to targeted attacks. To this end, it suffices to replace ${\max_{c \neq y} \log \mathcal{C}(\mathbf{x}')_{c}}$ in Eq. \eqref{eq:CW_loss} with ${\log \mathcal{C}(\mathbf{x}')_{t}}$, where ${t}$ shows the target class output. In this paper, we only consider un-targeted attacks.}
\begin{equation}\label{eq:CW_loss}
	\mathcal{L}(\mathbf{x}')=\max\big(0, \log \mathcal{C}(\mathbf{x}')_{y} - \max_{c \neq y} \log \mathcal{C}(\mathbf{x}')_{c}\big),	
\end{equation}  
where ${\mathcal{C}(\mathbf{x}')_{y}}$ indicates the ${y}$-th element of the classifier output.
In \textcolor{edit}{the} C\&W objective, we always have ${\mathcal{L}(\mathbf{x}') \geq 0}$.
The minimum occurs when ${\mathcal{C}(\mathbf{x}')_{y} \leq \max_{c \neq y} \mathcal{C}(\mathbf{x}')_{c}}$, which is an indication that our classifier has been fooled.
Thus, finding an adversarial example for the input data ${\mathbf{x}}$ can be written as~\citep{li2019nattack}:
\begin{equation}\label{eq:adv_example}
{\mathbf{x}}_{adv}=\argmin_{\mathbf{x}' \in \mathcal{S}(\mathbf{x})} \mathcal{L}(\mathbf{x}').
\end{equation}
Here, ${\mathcal{S}(\mathbf{x})}$ denotes a set that contains similar data to ${\mathbf{x}}$ in an appropriate manner.
For example, it is common to define
\begin{equation}\label{eq:set}
	\mathcal{S}(\mathbf{x}) = \left\{\mathbf{x}' \in \mathcal{X}^{d}~\big\rvert~\norm{\mathbf{x}' - \mathbf{x}}_{p}\leq \epsilon_{\max} \right\}
\end{equation}
for image data.
In this paper, we define ${\mathcal{S}(\mathbf{x})}$ as in Eq.~\eqref{eq:set} since we deal with the application of our attack on images.

\subsection{\textcolor{edit}{Flow-based Modeling}}\label{sec:background}


\paragraph{Normalizing Flows.}\label{sec:sec:normalizing_flows}

Normalizing flows~(NF)~\citep{tabak2013family, dinh2015nice, rezende2015variational} are a family of generative models that aim at modeling the probability distribution of a given dataset.
To this end, they make use of the well-known \textit{change of variables} formula.
In particular, let ${\mathbf{Z} \in \mathbb{R}^{d}}$ denote a random vector with a straightforward, known distribution such as uniform or standard normal.
The \textit{change of variables} formula states that if we apply an invertible and differentiable transformation ${\mathbf{f}\left(\cdot\right): \mathbb{R}^{d} \rightarrow \mathbb{R}^{d}}$ on ${\mathbf{Z}}$ to obtain a new random vector ${\mathbf{X} \in \mathbb{R}^{d}}$, the relationship between their corresponding distributions can be written as:
\begin{equation}\label{eq:change_of_variable}
p(\mathbf{x}) = p(\mathbf{z})\left|\mathrm{det}\Big(\dfrac{\partial \mathbf{f}}{\partial \mathbf{z}}\Big)\right|^{-1}.
\end{equation} 
Here, ${p(\mathbf{x})}$ and ${p(\mathbf{z})}$ denote the probability distributions of ${\mathbf{X}}$ and ${\mathbf{Z}}$, respectively.
Moreover, the multiplicative term on the right-hand side is called the absolute value of the Jacobian determinant.
This term accounts for the changes in the volume of ${\mathbf{Z}}$ due to applying the transformation ${\mathbf{f}(\cdot)}$.
Flow-based methods model the transformation ${\mathbf{f}(\cdot)}$ using stacked layers of invertible neural networks~(INN).
They then apply this transformation on a \textit{base random vector} ${\mathbf{Z}}$ to model the data density.
In this paper, we assume that the base random vector has a standard normal distribution.

\paragraph{Maximum Likelihood Estimation.}

To fit the parameters of INNs to the i.i.d.~data observations ${\mathbf{x}_1,\mathbf{x}_2, \dots, \mathbf{x}_n}$, NFs use the following maximum likelihood objective~\citep{rezende2015variational}:
\begin{equation}\label{eq:ML}
{\boldsymbol{\theta}}^{*}=\argmax_{\boldsymbol{\theta}} \dfrac{1}{n} \sum_{i=1}^{n} \log p_{\boldsymbol{\theta}}\left(\mathbf{x}_i\right).
\end{equation}
Here, ${\boldsymbol{\theta}}$ denotes the parameter set of the model and $p_{\boldsymbol{\theta}}$ is the density defined in Eq.~\eqref{eq:change_of_variable}. 
Note that INNs should be modeled such that they allow for efficient computation of their Jacobian determinant.
Otherwise, this issue can impose a severe hindrance in the application of NFs to high-dimensional data given the cubic complexity of determinant computation.
For a more detailed review of normalizing flows, we refer the interested reader to~\citep{papamakarios2019normalizing, kobyzev2019nfs} and the references within.

\paragraph{Training the Flow-based Models.}

We assume that we have access to some training data \textcolor{edit}{of the same domain} to pre-train our flow-based model.
However, at the time of adversarial example generation, we use unseen test data.
Note that while in our experiments we use the same training data as the classifier itself, we are not obliged to do so.
We observed that our results remain almost the same even if we separate the flow-based model training data from what the classifier is trained on.
We argue that using the same training data is valid since our flow-based models do not extract discriminative features, and they are only trained on clean data.
This is in contrast to other generative approaches used for adversarial example generation~\citep{baluja2018learning, xiao2018generating, wang2019adirect, huang2019black}.
\textcolor{edit2}{In Appendix~\ref{sec:ap:training_data}, we empirically show that not only is this statement accurate, but we can get almost the same performance by using similar datasets to the true one.}

\subsection{Natural Evolution Strategies and ${\mathcal{N}}$\textsc{Attack}}\label{sec:sec:NES_NATTACK}

\paragraph{Natural Evolution Strategies~(NES).}

Our goal is to solve the optimization problem of Eq.~\eqref{eq:adv_example} in a black-box setting, meaning that we only have access to the inputs and outputs of the classifier ${\mathcal{C}(\cdot)}$.
Natural Evolution Strategies~(NES) use the idea of \textit{search gradients} to optimize Eq.~\eqref{eq:adv_example}~\citep{wierstra2008nes, wierstra2014natural}.
To this end, a so-called \textit{search distribution} is first defined, and then the expected value of the original objective is optimized under this distribution.

In particular, let ${p(\mathbf{x'}|\boldsymbol{\psi})}$ denote the search distribution with parameters ${\boldsymbol{\psi}}$.
Then, in NES we aim to minimize
\begin{equation}\label{eq:nes}
	J(\boldsymbol{\psi})=\mathbb{E}_{p(\mathbf{x'}|\boldsymbol{\psi})}\left[\mathcal{L}(\mathbf{x}')\right]
\end{equation}  
over ${\boldsymbol{\psi}}$ as a surrogate for ${\mathcal{L}(\mathbf{x}')}$~\citep{wierstra2014natural}.
To minimize Eq.~\eqref{eq:nes} using gradient descent, one needs to compute the Jacobian of ${J(\boldsymbol{\psi})}$ with respect to ${\boldsymbol{\psi}}$.
To this end, NES makes use of the ``log-likelihood trick''~\citep{wierstra2014natural} (see Appendix~\ref{sec:ap:ll_trick})
\begin{align}\label{eq:jac_J}
	\nabla_{\boldsymbol{\psi}}J(\boldsymbol{\psi}) = \mathbb{E}_{p(\mathbf{x'}|\boldsymbol{\psi})}\left[\mathcal{L}(\mathbf{x}') \nabla_{\boldsymbol{\psi}}\log\big(p(\mathbf{x'}|\boldsymbol{\psi})\big)\right].
\end{align}
Finally, the parameters of the model are updated using a gradient descent step with learning rate ${\alpha}$:\footnote{Note that this update procedure is not what NES precisely stands for.
	It is rather a \textit{canonical gradient search algorithm} as is called by~\citet{wierstra2014natural}, which only makes use of a \textit{vanilla} gradient~\citep{wierstra2008nes} for evolution strategies.
	In fact, the \textit{natural} term in \textit{natural evolution strategies} represents an update of the form ${\boldsymbol{\psi} \leftarrow \boldsymbol{\psi} - \alpha \tilde{\nabla}_{\boldsymbol{\psi}}J(\boldsymbol{\psi})}$, where  ${\tilde{\nabla}_{\boldsymbol{\psi}}J(\boldsymbol{\psi}) = \mathbf{F}^{-1}\nabla_{\boldsymbol{\psi}}J(\boldsymbol{\psi})}$ is called the \textit{natural gradient}.
	Here, the matrix ${\mathbf{F}}$ is the \textit{Fischer information matrix} of the search distribution ${p(\mathbf{x'}|\boldsymbol{\psi})}$.
	However, since NES in the adversarial learning literature~\citep{ilyas2018black, ilyas2019prior, li2019nattack} points to Eq.~\eqref{eq:nes_update}, we use the same convention here.} 
\begin{equation}\label{eq:nes_update}
	\boldsymbol{\psi} \leftarrow \boldsymbol{\psi} - \alpha \nabla_{\boldsymbol{\psi}}J(\boldsymbol{\psi}).
\end{equation}

\paragraph{${\mathcal{N}}$\textsc{Attack}.}

To find an adversarial example for an input ${\mathbf{x}}$, $\mathcal{N}$\textsc{Attack}~\citep{li2019nattack} tries to find a distribution ${p(\mathbf{x'}|\boldsymbol{\psi})}$ over the set of legitimate adversaries ${\mathcal{S}(\mathbf{x})}$ in Eq.~\eqref{eq:set}.
Therefore, it models $\mathbf{x}'$ as
\begin{equation}\label{eq:NATTACK}
	\mathbf{x}'=\mathrm{proj}_{\mathcal{S}}\big(\tfrac{1}{2}(\tanh(\mathbf{z}) + 1)\big),
\end{equation}
where $\mathbf{z} \sim \mathcal{N}(\mathbf{z}|\boldsymbol{\mu}, \sigma^2 I)$ is an isometric normal distribution with mean~$\boldsymbol{\mu}$ and standard deviation~$\sigma$.
Moreover, $\mathrm{proj}_{\mathcal{S}}(\cdot)$ projects its input back into the set of legitimate adversaries ${\mathcal{S}(\mathbf{x})}$.
\citet{li2019nattack} define their model parameters as ${\boldsymbol{\psi}=\{\boldsymbol{\mu}, \sigma\}}$. 
Then, they find ${\sigma}$ using grid-search and ${\boldsymbol{\mu}}$ by the update rule of Eq.~\eqref{eq:nes_update} exploiting NES.

\subsection{\textcolor{edit}{Our Approach:} AdvFlow}\label{sec:sec:AdvFlow}

Recently, there has been some effort to detect adversarial examples from clean data.
The primary assumption of these methods is often that the adversaries come from a different distribution than the data itself; for instance, see \citep{ma2018characterizing, lee2018simple, zisselman2020resflow}.
Thus, to come up with more powerful adversarial attacks, it seems reasonable to construct adversaries that have a similar distribution to the clean data.
To this end, we propose \textit{AdvFlow}: a black-box adversarial attack that seeks to build inconspicuous adversaries by leveraging the power of normalizing flows~(NF) in exact likelihood modeling of the data~\citep{dinh2016density}.

Let ${\mathbf{f}(\cdot)}$ denote a pre-trained, invertible and differentiable NF model on the clean training data.
To reach our goal of decreasing the attack's detectability, we propose using this pre-trained, fixed NF transformation to model the adversaries.
In an analogy with Eq.~\eqref{eq:NATTACK}, we assume that our adversarial example comes from a distribution that is modeled by
\begin{equation}\label{eq:advflow_dist}
	\mathbf{x}'=\mathrm{proj}_{\mathcal{S}}\big(\mathbf{f}(\mathbf{z})\big),\qquad\mathbf{z}\sim\mathcal{N}(\mathbf{z}|\boldsymbol{\mu}, \sigma^2 I)
\end{equation}
where $\mathrm{proj}_{\mathcal{S}}(\cdot)$ is a projection rule that keeps the generated examples in the set of legitimate adversaries ${\mathcal{S}(\mathbf{x})}$.
By the change of variables formula from Eq.~\eqref{eq:change_of_variable}, we know that ${\mathbf{f}(\mathbf{z})}$ in Eq.~\eqref{eq:advflow_dist} is distributed similar to the clean data distribution.
The only difference is that the base density is transformed by an affine mapping, i.e., from $\mathcal{N}(\mathbf{z}|\mathbf{0}, I)$ to $\mathcal{N}(\mathbf{z}|\boldsymbol{\mu}, \sigma^2 I)$.
This small adjustment can result in an overall distribution for which the generated samples are likely to be adversarial.

Putting the rule of the \textit{lazy statistician}~\citep{wasserman2013all} together with our attack definition in Eq.~\eqref{eq:advflow_dist}, we can write down the objective function of Eq.~\eqref{eq:nes} as
\begin{equation}\label{eq:advflow_cost}
	J(\boldsymbol{\mu}, \sigma)=\mathbb{E}_{p(\mathbf{x'}|\boldsymbol{\mu}, \sigma)}\left[\mathcal{L}(\mathbf{x}')\right]
					    =\mathbb{E}_{\mathcal{N}(\mathbf{z}|\boldsymbol{\mu}, \sigma^2 I)}\left[\mathcal{L}\bigg(\mathrm{proj}_{\mathcal{S}}\big(\mathbf{f}(\mathbf{z})\big)\bigg)\right].
\end{equation}
As in $\mathcal{N}$\textsc{Attack}~\citep{li2019nattack}, we will consider ${\sigma}$ to be a hyperparameter.\footnote{Indeed, we can also optimize $\sigma$ alongside ${\boldsymbol{\mu}}$ to enhance the attack strength. However, since $\mathcal{N}$\textsc{Attack}~\citep{li2019nattack} only optimizes ${\boldsymbol{\mu}}$, we also stick with the same setting.}
Thus, we are only required to minimize ${J(\boldsymbol{\mu}, \sigma)}$ with respect to ${\boldsymbol{\mu}}$.
Using the ``log-likelihood trick'' of Eq.~\eqref{eq:jac_J}, we can derive the Jacobian of ${J(\boldsymbol{\mu}, \sigma)}$ as
\begin{equation}\label{eq:jac_advflow}
	\nabla_{\boldsymbol{\mu}}J(\boldsymbol{\mu}, \sigma) = \mathbb{E}_{\mathcal{N}(\mathbf{z}|\boldsymbol{\mu}, \sigma^2 I)}\left[\mathcal{L}\bigg(\mathrm{proj}_{\mathcal{S}}\big(\mathbf{f}(\mathbf{z})\big)\bigg)	\nabla_{\boldsymbol{\mu}}\log \mathcal{N}(\mathbf{z}|\boldsymbol{\mu}, \sigma^2 I)\right].
\end{equation}
This expectation can then be estimated by sampling from a distribution ${\mathcal{N}(\mathbf{z}|\boldsymbol{\mu}, \sigma^2 I)}$ and forming their sample average.
Next, we update the parameter ${\boldsymbol{\mu}}$ by performing a gradient descent step
\begin{equation}\label{eq:advflow_update}
\boldsymbol{\mu} \leftarrow \boldsymbol{\mu} - \alpha \nabla_{\boldsymbol{\mu}}J(\boldsymbol{\mu}, \sigma).
\end{equation}
In the end, we generate our adversarial example by sampling from Eq.~\eqref{eq:advflow_dist}.

\paragraph{Practical Considerations.}

To help our model to start its search from an appropriate point, we first transform the clean data to its latent space representation.
Then, we aim to find a small additive latent space perturbation in the form of a normal distribution.
Moreover, as suggested in~\citep{li2019nattack}, instead of working with ${\mathcal{L}\big(\mathrm{proj}_{\mathcal{S}}\big(\mathbf{f}(\mathbf{z})\big)\big)}$ directly, we normalize them so that they have zero mean and unit variance to help AdvFlows convergence faster.
Finally, among different flow-based models, it is preferable to choose those that have a straightforward inverse, such as~\citep{dinh2016density, kingma2018glow, durkan2019neural, dolatabadi2020lrs}.
This way, we can efficiently go back and forth between the original data and their base distribution representation. 
Algorithm~\ref{alg:advflow} in Appendix~\ref{sec:ap:advflow} summarizes our black-box attack method.
\textcolor{edit2}{Other variations of AdvFlow can also be found in Appendix~\ref{sec:ap:algs}.
These variations include our solution to high-resolution images and investigation of un-trained AdvFlow.}

\subsubsection{AdvFlow Interpretation}

We can interpret AdvFlow from two different perspectives.

First, there exists a probabilistic view: we use the flow-based model transformation of the original data, and then try to adjust it using an affine transformation on its base distribution.
The amount of this change is determined by our urge to minimize the C\&W cost of Eq.~\eqref{eq:CW_loss} such that we get the minimum value on average.
Thus, if it is successful, we will end up having a distribution whose samples are likely to be adversarial.
Meanwhile, since this distribution is initialized with that of clean data, it resembles the clean data density closely.

Second, we can think of AdvFlow as a search over the latent space of flow-based models.
We map the clean image to the latent space and then try to search in the vicinity of that point to find an adversarial example.
This search is ruled by the objective of Eq.~\eqref{eq:advflow_cost}.
Since our approach exploits a fully invertible, pre-trained flow-based model, we would expect to get an adversarial example that resembles the original image in the structure and look less noisy.
This adjustment gives our model the flexibility to produce perturbations that take the structure of clean data into account (see Figure~\ref{fig:celeba_example}).

\subsubsection{Uniqueness of AdvFlow Perturbations}\label{sec:sec:advflow_lemma}

In this section, we present a lemma about the nature of perturbations generated by AdvFlow and $\mathcal{N}$\textsc{Attack}~\citep{li2019nattack}.
As a direct result of this lemma, we can easily deduce that the adversaries generated by AdvFlow can be approximated by a normal distribution whose components are dependent.
However, this is not the case for $\mathcal{N}$\textsc{Attack} as they always have independent elements.
In this sense, we can then rigorously conclude that the AdvFlow perturbations are unique and cannot be generated by $\mathcal{N}$\textsc{Attack}.
Thus, we cannot expect $\mathcal{N}$\textsc{Attack} to be able to generate perturbations that look like the original data. 
This result can also be generalized to many other attack methods as they often use an additive, independent perturbation. Proofs can be found in Appendix~\ref{sec:ap:proofs}.

\begin{lemma}\label{lemma}
	Let ${\mathbf{f}(\mathbf{x})}$ be an invertible, differentiable function.
	For a small perturbation ${\boldsymbol{\delta}_z}$ we have
	\begin{equation}\label{eq:lemma}\nonumber
		\boldsymbol{\delta} = \mathbf{f}\big(\mathbf{f}^{-1}(\mathbf{x}) + \boldsymbol{\delta}_z\big) - \mathbf{x}\approx\big(\nabla\mathbf{f}^{-1}(\mathbf{x})\big)^{-1}\boldsymbol{\delta}_z.
	\end{equation}
\end{lemma}

\begin{corollary}\label{corollary}
	The adversarial perturbations generated by AdvFlow have dependent components.
	In contrast, ${\mathcal{N}}$\textsc{Attack} perturbation components are independent.
\end{corollary}
 
\section{Experimental Results}\label{sec:experimental_results}

In this section, we present our experimental results.
First, we see how the adversarial examples generated by the proposed model can successfully mislead adversarial example detectors.
Then, we show the attack success rate and the number of queries required to attack both vanilla and defended models.
Finally, we examine the transferability of the generated attacks between defended classifiers.
To see the details of the experiments, please refer to Appendix~\ref{sec:ap:imp_det}.
Also, more simulation results can be found in Appendices~\ref{sec:ap:ext_sim_res} and~\ref{sec:ap:algs}.

For each dataset, we pre-train a flow-based model and fix it across the experiments.
To this end, we use a modified version of Real~NVP~\citep{dinh2016density} as introduced in~\citep{ardizzone2019guided}, the details of which can be found in Appendix~\ref{sec:ap:nf_arch}.
Once trained, we then try to attack target classifiers in a black-box setting using AdvFlow~(Algorithm~\ref{alg:advflow}).

\subsection{Detectability}\label{sec:sec:detectability}

One approach to defend pre-trained classifiers is to employ adversarial example detectors.
This way, a detector is trained and put on top of the classifier.
Before feeding inputs to the un-defended classifier, every input has to be checked whether it is adversarial or not.
One common assumption among such detectors is that the adversaries come from a different distribution than the clean data~\citep{ma2018characterizing, lee2018simple}.
Thus, the performance of these detectors seems to be a suitable measure to quantify the success of our model in generating adversarial examples that have the same distribution as the original data.
To this end, we choose LID~\citep{ma2018characterizing}, Mahalanobis~\citep{lee2018simple}, and Res-Flow~\cite{zisselman2020resflow} adversarial attack detectors to assess the performance of the proposed approach. We compare our results with $\mathcal{N}$\textsc{Attack}~\citep{li2019nattack} \textit{that also approaches the black-box adversarial attack from a distributional perspective} for a fair comparison.
As an ablation study, we also consider the un-trained version of AdvFlows where the weights of the NF models are set randomly.
This way, we can observe the effect of the clean data distribution in misleading adversarial example detectors more precisely. 
We first generate a set of adversarial examples alongside some noisy ones using the test set.
Then, we use $10\%$ of the adversarial, noisy, and clean image data to train adversarial attack detectors.
Details of our experiments in this section can be found in Appendix~\ref{sec:ap:mah_det}.

We report the area under the receiver operating characteristic curve~(AUROC) and the detection accuracy for each case in Table~\ref{tab:Detectability}.
As seen, in almost all the cases the selected adversarial detectors struggle to detect the attacks generated by AdvFlow in contrast to $\mathcal{N}$\textsc{Attack}.
These results support our statement earlier about the distribution of the attacks being more similar to that of data, hence the failure of adversarial example detectors.
Also, we see that pre-training the AdvFlow using clean data is crucial in fooling adversarial example detectors.   

Finally, Figure~\ref{fig:latent_separation} shows the relative change in the base distribution of the flow-based model for adversarial examples of Table~\ref{tab:Detectability}.
Interestingly, we see that AdvFlow adversaries are distinctively closer to the clean data compared to $\mathcal{N}$\textsc{Attack}~\cite{li2019nattack}.
These results highlight the need to reconsider the underlying assumption that adversaries come from a different distribution than the clean data.
Also, it can motivate training classifiers that learn data distributions, as our results reveal this is not currently the case.
\begin{table*}[tb!]
	\caption{Area under the receiver operating characteristic curve~(AUROC) and accuracy of detecting adversarial examples generated by $\mathcal{N}$\textsc{Attack}~\citep{li2019nattack} and AdvFlow (un. for un-trained and tr. for pre-trained NF) using LID~\citep{ma2018characterizing}, Mahalanobis~\citep{lee2018simple}, and Res-Flow~\cite{zisselman2020resflow} adversarial attack detectors.
	In each case, the classifier has a ResNet-34~\cite{he2016deep} architecture.}
	\label{tab:Detectability}
	
	\begin{center}
		\begin{small}
			\resizebox{\columnwidth}{!}{
				\begin{tabular}{cccccccc}
					\toprule
					\parbox[t]{2mm}{\multirow{2}{*}{\rotatebox[origin=c]{90}{Data}}}
					&Metric                                      & \multicolumn{3}{c}{AUROC(\%) \textuparrow}   & \multicolumn{3}{c}{Detection Acc.(\%) \textuparrow}\\
					\cmidrule(lr){3-5}\cmidrule(lr){6-8}
					&Method                                      & $\mathcal{N}$\textsc{Attack}    & AdvFlow (un.)  & AdvFlow (tr.)   & $\mathcal{N}$\textsc{Attack}    & AdvFlow (un.)  & AdvFlow (tr.)\\
					\midrule
					\parbox[t]{2mm}{\multirow{3}{*}{\rotatebox[origin=c]{90}{\scriptsize{CIFAR-10}}}}
					&LID~\citep{ma2018characterizing}            & $78.69$ & $84.39$ & $\mathbf{57.59}$    & $72.12$ & $77.11$ & $\mathbf{55.74}$\\
					&Mahalanobis~\citep{lee2018simple}           & $97.95$ & $99.50$ & $\mathbf{66.85}$    & $95.59$ & $97.46$ & $\mathbf{62.21}$\\
					&Res-Flow~\citep{zisselman2020resflow}       & $97.90$ & $99.40$ & $\mathbf{67.03}$    & $94.55$ & $97.21$ & $\mathbf{62.60}$\\
					\midrule
					\parbox[t]{2mm}{\multirow{3}{*}{\rotatebox[origin=c]{90}{\footnotesize{SVHN}}}}
					&LID~\citep{ma2018characterizing}            & $\mathbf{57.70}$  & $58.92$ &$61.11$    & $\mathbf{55.60}$  & $56.43$ & $58.21$\\
					&Mahalanobis~\citep{lee2018simple}           & $73.17$ & $74.67$ & $\mathbf{64.72}$    & $68.20$ & $69.46$ & $\mathbf{60.88}$\\
					&Res-Flow~\citep{zisselman2020resflow}       & $69.70$ & $74.86$ & $\mathbf{64.68}$    & $64.53$ & $68.41$ & $\mathbf{61.13}$\\
					\bottomrule
				\end{tabular}}
		\end{small}
	\end{center}
	
\end{table*}

\begin{figure}[tb!]
	\centering
	\begin{subfigure}{.44\textwidth}
		\centering
		\includegraphics[width=1.0\textwidth]{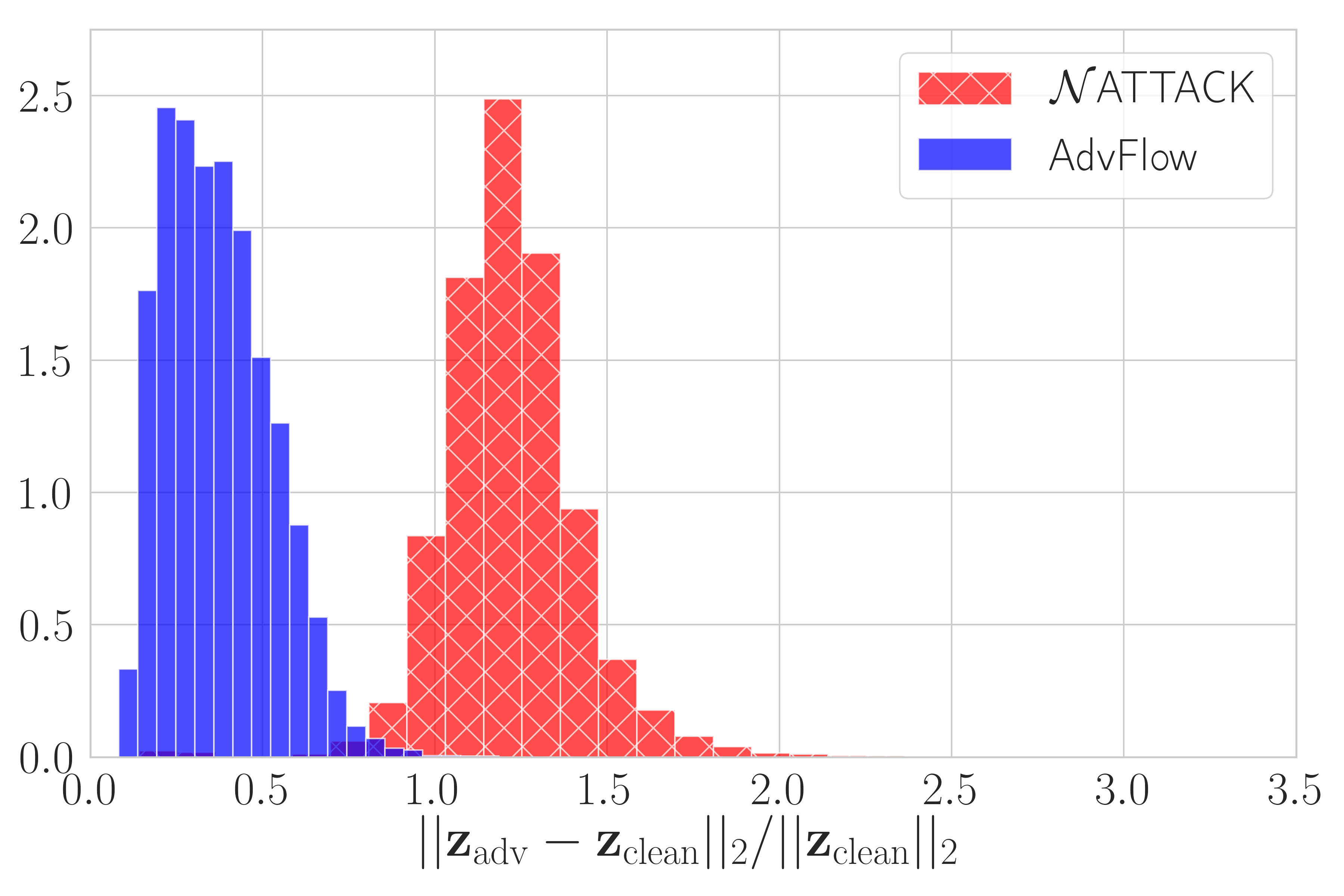}
	\end{subfigure}\hspace*{0.2em}
	\begin{subfigure}{.44\textwidth}
		\centering
		\includegraphics[width=1.0\textwidth]{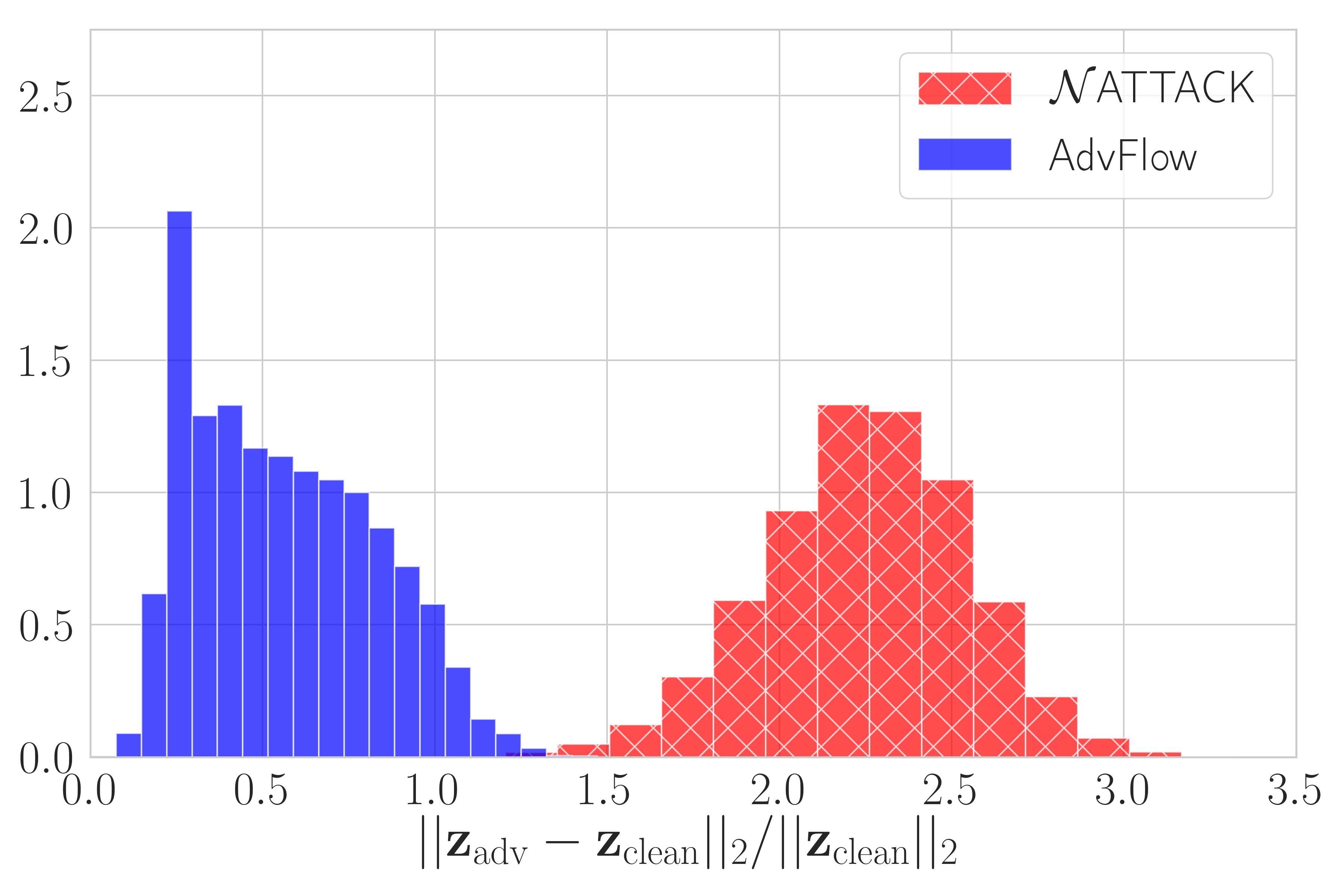}
	\end{subfigure}
	\caption{Relative change in the base distribution of the flow-based model for adversarial examples generated by AdvFlow and $\mathcal{N}$\textsc{Attack} for CIFAR-10~\cite{krizhevsky2009learning} (left) and SVHN~\cite{netzer2011reading} (right) classifiers of Table~\ref{tab:Detectability}.}
	\label{fig:latent_separation}
\end{figure}

\subsection{Success Rate and Number of Queries}\label{sec:sec:succ_query}

Next, we investigate the performance of the proposed model in attacking vanilla and defended image classifiers.
It was shown previously that $\mathcal{N}$\textsc{Attack} struggles to break into adversarially trained models more than any other defense~\citep{li2019nattack}.
Thus, we select some of the most recent defense techniques that are built upon adversarial training~\citep{madry2018towards}.
This selection also helps us in quantifying the attack transferability in our next experiment.
Therefore, we select Free~\citep{shafahi2019free} and Fast~\citep{wong2020fast} adversarial training, alongside adversarial training with auxiliary rotations~\cite{hendrycks2019using} as the defense mechanisms that our classifiers employ.
Note that these models are the most recent defenses built upon adversarial training.
For a brief explanation of each one of these methods, please refer to Appendix~\ref{sec:ap:defense}.
We then train target classifiers on CIFAR-10~\citep{krizhevsky2009learning} and SVHN~\citep{netzer2011reading} datasets.
The architecture that we use here is the well-known Wide-ResNet-32~\cite{zagoruyko2016wresnet} with width $10$.
We then try to attack these classifiers by generating adversarial examples on the test set.
We compare our proposed model with bandits with time and data-dependent priors~\citep{ilyas2019prior}, $\mathcal{N}$\textsc{Attack}~\citep{li2019nattack}, and SimBA~\cite{guo2019simba}. \footnote{Note that SimBA~\citep{guo2019simba} is originally designed for efficient $\ell_2$ attacks, and it may not use the entire $10,000$ query quota for small images. Anyways, we included SimBA~\citep{guo2019simba} in the paper as per one of the reviewers' suggestions.}
To simulate a realistic environment, we set the maximum number of queries to $10,000$.
Moreover, for $\mathcal{N}$\textsc{Attack} and AdvFlow we use a population size of $20$.
More details on the defense methods as well as attack hyperparameters can be found in Appendices~\ref{sec:ap:defense} and~\ref{sec:ap:attack_hyper}.

Tables~\ref{tab:success} and~\ref{tab:query}  show the success rate as well as the average and median number of queries required to successfully attack a vanilla/defended classifier.
Also, Figure~\ref{fig:ap_succ_query} in Appendix~\ref{sec:ap:ext_sim_res} shows the attack success rate for AdvFlow and $\mathcal{N}$\textsc{Attack}~\citep{li2019nattack} versus the maximum number of queries for defended models.
As can be seen, AdvFlows can improve upon the performance of $\mathcal{N}$\textsc{Attack}~\citep{li2019nattack} in all of the defended models in terms of the number of queries and attack success rate.
Also, it should be noted that although our performance on vanilla classifiers is worse than $\mathcal{N}$\textsc{Attack}~\citep{li2019nattack}, we are still generating adversaries that are not easily detectable by adversarial example detectors and come from a similar distribution to the clean data.

\begin{table*}[tb!]
	\caption{Attack success rate of black-box adversarial attacks on CIFAR-10~\citep{krizhevsky2009learning} and SVHN~\citep{netzer2011reading} Wide-ResNet-32~\cite{zagoruyko2016wresnet} classifiers.
	All attacks are with respect to $\ell_{\infty}$ norm with $\epsilon_{\max}=8/255$.}
	\label{tab:success}
	\begin{center}
		\begin{small}
				\begin{tabular}{ccccccc}
					\toprule
					\parbox[t]{2mm}{\multirow{2}{*}{\rotatebox[origin=c]{90}{Data}}} &&& \multicolumn{4}{c}{Success Rate(\%) \textuparrow}\\
					\cmidrule(lr){4-7}
					&Defense                                & Acc.(\%)                & Bandits\citep{ilyas2019prior}  & $\mathcal{N}$\textsc{Attack}~\citep{li2019nattack} & SimBA~\citep{guo2019simba} & AdvFlow \\
					\midrule
					\parbox[t]{2mm}{\multirow{4}{*}{\rotatebox[origin=c]{90}{CIFAR-10}}}
					&Vanilla~\citep{zagoruyko2016wresnet}   & $91.77$                 & $98.81$	                       & $\mathbf{100}$                                     & $99.99$                    & $99.42$\\
					&FreeAdv~\citep{shafahi2019free}        & $81.29$                 & $37.12$	                       & $38.97$	                                        & $35.52$	                 & $\mathbf{41.21}$\\
					&FastAdv~\citep{wong2020fast}           & $86.33$                 & $36.60$                        & $36.90$                                           	& $35.07$	                 & $\mathbf{40.22}$\\
					&RotNetAdv~\citep{hendrycks2019using}   & $86.58$                 & $37.73$	                       & $38.04$	                                        & $35.63$	                 & $\mathbf{40.67}$\\
					\midrule
					\parbox[t]{2mm}{\multirow{4}{*}{\rotatebox[origin=c]{90}{SVHN}}}
					&Vanilla~\citep{zagoruyko2016wresnet}   & $96.45$                 & $87.84$	                       & $\mathbf{98.76}$	                                & $97.26$                    & $90.31$\\
					&FreeAdv~\citep{shafahi2019free}        & $86.47$                 & $49.64$	                       & $50.28$	                                        & $46.28$	                 & $\mathbf{50.76}$\\
					&FastAdv~\citep{wong2020fast}           & $93.90$                 & $40.43$	                       & $35.42$	                                        & $36.19$	                 & $\mathbf{41.49}$\\
					&RotNetAdv~\citep{hendrycks2019using}   & $90.33$                 & $43.47$	                       & $41.49$                                            & $39.01$	                 & $\mathbf{44.22}$\\
					\bottomrule
				\end{tabular}
		\end{small}
	\end{center}
\end{table*}

\begin{table*}[tb!]
	\caption{Average (median) of the number of queries needed to generate an adversarial example for CIFAR-10~\citep{krizhevsky2009learning} and SVHN~\citep{netzer2011reading} Wide-ResNet-32~\cite{zagoruyko2016wresnet} classifiers of Table~\ref{tab:success}.
		For a fair comparison, we first find the samples where all the attack methods are successful, and then compute the average (median) of queries for these samples. 
		Note that for $\mathcal{N}$\textsc{Attack} and AdvFlow we check whether we arrived at an adversarial point every $200$ queries, and hence, the medians are a multiples of $200$.}
	\label{tab:query}
	\begin{center}
		\begin{small}
			\begin{tabular}{cccccc}
				\toprule
				\parbox[t]{2mm}{\multirow{2}{*}{\rotatebox[origin=c]{90}{Data}}} && \multicolumn{4}{c}{Query Average (Median) on Mutually Successful Attacks \textdownarrow}\\
				\cmidrule(lr){3-6}
				&Defense                                & Bandits\citep{ilyas2019prior}    & $\mathcal{N}$\textsc{Attack}~\citep{li2019nattack} & SimBA~\citep{guo2019simba}        & AdvFlow \\
				\midrule
				\parbox[t]{2mm}{\multirow{4}{*}{\rotatebox[origin=c]{90}{CIFAR-10}}}
				&Vanilla~\citep{zagoruyko2016wresnet}   & $552.69~(182)$	               & $\mathbf{237.58}~(200)$	                        & $237.70~(\mathbf{126})$           & $949.31~(400)$\\
				&FreeAdv~\citep{shafahi2019free}        & $1062.7~(354)$	               & $874.91~(400)$                                     & $463.09~(244)$	                & $\mathbf{421.63}~(\mathbf{200})$\\
				&FastAdv~\citep{wong2020fast}           & $1065.92~(358)$                  & $973.05~(400)$                                     & $\mathbf{428.81}~(234)$           & $436.8~(\mathbf{200})$\\
				&RotNetAdv~\citep{hendrycks2019using}   & $1085.43~(408)$                  & $941.67~(400)$                                     & $471.99~(259)$                    & $\mathbf{424.95}~(\mathbf{200})$\\
				\midrule
				\parbox[t]{2mm}{\multirow{4}{*}{\rotatebox[origin=c]{90}{SVHN}}}
				&Vanilla~\citep{zagoruyko2016wresnet}   & $1750.65~(1128)$                 & $408.75~(200)$                                   	& $\mathbf{202.07}~(\mathbf{107})$  & $1572.24~(600)$\\
				&FreeAdv~\citep{shafahi2019free}        & $819.98~(250)$                   & $903.12~(400)$                                     & $\mathbf{365.42}~(216)$           & $692.73~(\mathbf{200})$\\
				&FastAdv~\citep{wong2020fast}           & $755.23~(284)$                   & $1243.38~(600)$	                                & $\mathbf{307.73}~(216)$           & $526.37~(\mathbf{200})$\\
				&RotNetAdv~\citep{hendrycks2019using}   & $663.07~(202)$                   & $756.48~(400)$	                                    & $\mathbf{319.93}~(\mathbf{186})$  & $480.02~(200)$\\
				\bottomrule
			\end{tabular}
		\end{small}
	\end{center}
\end{table*}

\subsection{Transferability}\label{sec:sec:transferability}

Finally, we examine the transferability of the generated attacks for each of the classifiers in Table~\ref{tab:success}.
In other words, we generate attacks using a substitute classifier, and then try to attack another target model.
The results of this experiment are shown in Figure~\ref{fig:ap_confusion_trans} of Appendix~\ref{sec:ap:ext_sim_res}.
As seen, the generated attacks by AdvFlow transfer to other defended models easier than the vanilla one.
This observation precisely matches our intuition about the mechanics of AdvFlow.
More specifically, we know that in AdvFlow the model is learning a distribution that is more expressive than the one used by $\mathcal{N}$\textsc{Attack}.
Also, we have seen in Section~\ref{sec:sec:advflow_lemma} that the perturbations generated by AdvFlow have dependent elements in contrast to $\mathcal{N}$\textsc{Attack}.
As a result, AdvFlow learns to attack classifiers using higher-level features (Figure~\ref{fig:celeba_example}).  
Thus, since vanilla classifiers use different features for classification than the defended ones, AdvFlows are less transferable from defended models to vanilla ones.
In contrast, the expressiveness of AdvFlows enables the attacks to be transferred more successfully between adversarially trained classifiers, and from vanilla to defended ones.
  
\section{Conclusion and Future Directions}\label{sec:conclusion}

In this paper, we introduced AdvFlow: a novel adversarial attack model that utilizes the capacity of normalizing flows in representing data distributions.
We saw that the adversarial perturbations generated by the proposed approach can be approximated using normal distributions with dependent components.
In this sense, $\mathcal{N}$\textsc{Attack}~\citep{li2019nattack} cannot generate such adversaries.
As a result, AdvFlows are less conspicuous to adversarial example detectors in contrast to their $\mathcal{N}$\textsc{Attack}~\citep{li2019nattack} counterpart.
This success is due to AdvFlow being pre-trained on the data distribution, resulting in adversaries that look like the clean data.
We also saw the capability of the proposed method in improving the performance of bandits~\cite{ilyas2019prior}, $\mathcal{N}$\textsc{Attack}~\citep{li2019nattack}, and SimBA~\cite{guo2019simba} on adversarially trained classifiers.
This improvement is in terms of both attack success rate and the number of queries.

Flow-based modeling is an active area of research.
There are numerous extensions to the current work that can be investigated upon successful expansion of normalizing flow models in their range and power.
For example, while $\mathcal{N}$\textsc{Attack}~\citep{li2019nattack} and other similar approaches~\citep{ilyas2018black, ilyas2019prior} are specifically designed for use on image data, the current work can potentially be expanded to entail other forms of data such as graphs~\citep{liu2019gnf, shi2020graphaf}.
Also, since normalizing flows can effectively model probability distributions, finding the distribution of well-known perturbations may lead to increasing classifier robustness against adversarial examples.
We hope that this work can provide a stepping stone to exploiting such powerful models for adversarial machine learning.

\clearpage  
\section*{Broader Impact}
In this paper, we introduce a novel adversarial attack algorithm called AdvFlow.
It uses pre-trained normalizing flows to generate adversarial examples.
This study is crucial as it indicates the vulnerability of deep neural network~(DNN) classifiers to adversarial attacks.

More precisely, our study reveals that the common assumption made by adversarial example detectors (such as the Mahalanobis detector~\citep{lee2018simple}) that the adversaries come from a different distribution than the data may not be an accurate one.
In particular, we show that we can generate adversaries that come from a close distribution to the data, yet they intend to mislead the classifier decision.
Thus, we emphasize that adversarial example detectors need to adjust their assumption about the distribution of adversaries before being deployed in real-world situations.

Furthermore, since our adversarial examples are closely related to the data distribution, our method shows that DNN classifiers are not learning to classify the data based on their underlying distribution.
Otherwise, they would have resisted the attacks generated by AdvFlow.
Thus, it can bring the attention of the machine learning community to training their DNN classifiers in a distributional sense.

All in all, we pinpoint a failure of DNN classifiers to the rest of the community so that they can become familiar with the limitations of the status-quo.
This study, and similar ones, could raise awareness among researchers about the real-world pitfalls of DNN classifiers, with the aim of consolidating them against such threats in the future.

\begin{ack}
We would like to thank the reviewers for their valuable feedback on our work, helping us to improve the final manuscript.
We also would like to thank the authors and maintainers of PyTorch~\citep{paszke2017automatic}, NumPy~\cite{harris2020array}, and Matplotlib~\cite{hunter2007matplotlib}.

This research was undertaken using the LIEF HPC-GPGPU Facility hosted at the University of Melbourne. This Facility was established with the assistance of LIEF Grant LE170100200.
\end{ack}

\bibliographystyle{apalike}
\bibliography{references}

\clearpage
\appendix
\textbf{\Large Supplementary Materials}

\textcolor{edit2}{This supplementary document includes the following content to support the material presented in the paper:
\begin{itemize}
    \item In Section~\ref{sec:ap:math}, we present the ``log-likelihood trick" and the proof to our lemma and corollary.
	\item In Section~\ref{sec:ap:imp_det}, we give the implementation details of our algorithm and experiments. Besides introducing the flow-based model architecture, we present a detailed explanation of classifier architectures and defense mechanisms used to evaluate our method. The set of hyperparameters used in each defense and attack model is also given.
	\item In Section~\ref{sec:ap:ext_sim_res}, we present an extended version of our simulation results. Moreover, we investigate the effect of the training data on our algorithm's performance.
	\item In Section~\ref{sec:ap:algs}, we see important extensions to the current work. These extensions include our solution to high-resolution images and un-trained AdvFlow. 
\end{itemize}}
 
\section{Mathematical Details}\label{sec:ap:math}

\subsection{The Log-likelihood Trick}\label{sec:ap:ll_trick}

Here we provide the complete proof of the ``log-likelihood trick'' as presented in~\cite{wierstra2014natural}:
\begin{align}\label{eq:ap:LLtrick}\nonumber
	\nabla_{\boldsymbol{\psi}}J(\boldsymbol{\psi})  &= \nabla_{\boldsymbol{\psi}}\mathbb{E}_{p(\mathbf{x'}|\boldsymbol{\psi})}\left[\mathcal{L}(\mathbf{x}')\right]\\\nonumber
													&= \nabla_{\boldsymbol{\psi}} \int \mathcal{L}(\mathbf{x}') p(\mathbf{x'}|\boldsymbol{\psi})\mathrm{d}\mathbf{x}'\\\nonumber
													&= \int \mathcal{L}(\mathbf{x}') \nabla_{\boldsymbol{\psi}}p(\mathbf{x'}|\boldsymbol{\psi})\mathrm{d}\mathbf{x}'\\\nonumber
													&= \int \mathcal{L}(\mathbf{x}') \frac{\nabla_{\boldsymbol{\psi}}p(\mathbf{x'}|\boldsymbol{\psi})}{p(\mathbf{x'}|\boldsymbol{\psi})}p(\mathbf{x'}|\boldsymbol{\psi})\mathrm{d}\mathbf{x}'\\\nonumber
													&= \int \mathcal{L}(\mathbf{x}') \nabla_{\boldsymbol{\psi}}\log\big(p(\mathbf{x'}|\boldsymbol{\psi})\big) p(\mathbf{x'}|\boldsymbol{\psi})\mathrm{d}\mathbf{x}'\\\nonumber
													&= \mathbb{E}_{p(\mathbf{x'}|\boldsymbol{\psi})}\left[\mathcal{L}(\mathbf{x}') \nabla_{\boldsymbol{\psi}}\log\big(p(\mathbf{x'}|\boldsymbol{\psi})\big)\right].
\end{align}

\subsection{Proofs}\label{sec:ap:proofs}

\begin{lemma}\label{lemma_ap}
	Let ${\mathbf{f}(\mathbf{x})}$ be an invertible, differentiable function.
	For a small perturbation ${\boldsymbol{\delta}_z}$ we have
	\begin{equation}\label{eq:ap:lemma}\nonumber
	\boldsymbol{\delta} = \mathbf{f}\big(\mathbf{f}^{-1}(\mathbf{x}) + \boldsymbol{\delta}_z\big) - \mathbf{x}\approx\big(\nabla\mathbf{f}^{-1}(\mathbf{x})\big)^{-1}\boldsymbol{\delta}_z.
	\end{equation}
\end{lemma}

\begin{proof}\label{lemma_proof_ap}
	By the first-order Taylor series for ${\mathbf{f}\left(\cdot\right): \mathbb{R}^{d} \rightarrow \mathbb{R}^{d}}$ we have
	\begin{equation}\label{eq:ap:taylor}\nonumber
	\mathbf{f}(\mathbf{z} + \boldsymbol{\delta}_z) \approx \mathbf{f}(\mathbf{z}) + \nabla\mathbf{f}(\mathbf{z})\boldsymbol{\delta}_z,
	\end{equation}
	where ${\nabla\mathbf{f}(\mathbf{z})}$ is the ${d \times d}$ Jacobian matrix of the function ${\mathbf{f}(\cdot)}$.
	Now, by substituting ${\mathbf{z} = \mathbf{f}^{-1}(\mathbf{x})}$, and using the \textit{inverse function theorem}~\citep{rudin1964principles}, we can write
	\begin{align}\label{eq:ap:inverse_func}\nonumber
	\mathbf{f}(\mathbf{f}^{-1}(\mathbf{x}) + \boldsymbol{\delta}_z) \approx \mathbf{f}(\mathbf{f}^{-1}(\mathbf{x})) + \big(\nabla\mathbf{f}^{-1}(\mathbf{x})\big)^{-1}\boldsymbol{\delta}_z
	\end{align}
	which gives us the result immediately.
\end{proof}

\begin{corollary}\label{corollary_ap}
	The adversarial perturbations generated by AdvFlow have dependent components.
	In contrast, ${\mathcal{N}}$\textsc{Attack} perturbation components are independent.
\end{corollary}

\begin{proof}\label{corollat_proof_ap}
	By Lemma~\ref{lemma} we know that
	\begin{equation}\nonumber
	\boldsymbol{\delta}\approx\big(\nabla\mathbf{f}^{-1}(\mathbf{x})\big)^{-1}\boldsymbol{\delta}_z,
	\end{equation}
	where ${\boldsymbol{\delta}_z}$ is distributed according to an isometric normal distribution.
	For $\mathcal{N}$\textsc{Attack}, we have ${\mathbf{f}(\mathbf{z})=\frac{1}{2}\big(\tanh(\mathbf{z}) + 1\big)}$.
	Thus, it can be shown that $\nabla\mathbf{f}^{-1}(\mathbf{x}) = \mathrm{diag}\big(2\mathbf{x}\odot(1-\mathbf{x})\big)$.
	As a result, the random vector $\boldsymbol{\delta}$ will be still normally distributed with a diagonal covariance matrix, and hence, have independent components.
	In contrast, we know that for an effective flow-based model $\nabla\mathbf{f}^{-1}(\mathbf{x})$ is not always diagonal.
	Otherwise, this means that our NF is simply a data-independent affine transformation.
	For example, in Real~NVP~\citep{dinh2016density} which we use, this matrix is a product of lower and upper triangular matrices.
	Hence, for a normalizing flow model $\mathbf{f}(\cdot)$ we have non-diagonal $\nabla\mathbf{f}^{-1}(\mathbf{x})$.
	Thus, it will make the random variable $\boldsymbol{\delta}$ normal with correlated (dependent) components.
\end{proof}

\section{Implementation Details}\label{sec:ap:imp_det}

In this section, we present the implementation details of our algorithm and experiments.
Note that all of the experiments were run using a single NVIDIA Tesla V100-SXM2-16GB GPU.

\subsection{Normalizing Flows}\label{sec:ap:nf_arch}

For the flow-based models used in AdvFlow, we implement Real~NVP~\cite{dinh2016density} by using the framework of~\citet{ardizzone2019guided}.\footnote{\href{https://github.com/VLL-HD/FrEIA}{github.com/VLL-HD/FrEIA}}
In particular, let ${\mathbf{z}=[\mathbf{z}_1, \mathbf{z}_2]}$ denote the input to one layer of a normalizing flow transformation.
If we denote the output of this layer by ${\mathbf{x}=[\mathbf{x}_1, \mathbf{x}_2]}$, then the Real~NVP transformation between the input and output of this particular layer can be written as:
 \begin{align}\label{eq:ap:realnvp_ardizzone}
	 \mathbf{x}_1 &= \mathbf{z}_1 \odot \exp\big(\mathbf{s}_1\left({\mathbf{z}_2}\right)\big) + \mathbf{t}_1\left({\mathbf{z}_2}\right) \nonumber \\
	 \mathbf{x}_2 &= \mathbf{z}_2 \odot \exp\big(\mathbf{s}_2\left({\mathbf{x}_1}\right)\big) + \mathbf{t}_2\left({\mathbf{x}_1}\right),\nonumber
 \end{align}
where ${\odot}$ is an element-wise multiplication.
The functions ${\mathbf{s}_{1,2}(\cdot)}$ and ${\mathbf{t}_{1,2}(\cdot)}$ are called the scaling and translation functions.
Since the invertibility of the transformation does not depend on these functions, they are implemented using ordinary neural networks.
To help with the stability of the transformation, \citet{ardizzone2019guided} suggest using a soft-clamp before passing the output of scaling networks ${\mathbf{s}_{1,2}(\cdot)}$ to exponential function.
This soft-clamp function is implemented by
\begin{equation}\label{eq:ap:soft_clamp}\nonumber
s_{\rm clamp} = \frac{2\alpha}{\pi}\arctan\big(\frac{s}{\alpha}\big),
\end{equation}
where ${\alpha}$ is a hyperparameter that controls the amount of softening.
In our experiments, we set ${\alpha = 1.5}$. 
Moreover, at the end of each layer of transformation, we permute the output so that we end up getting different partitions of the data as $\mathbf{z}_1$ and $\mathbf{z}_2$.
The pattern by which the data is permuted is set at random at the beginning of the training process and kept fixed onwards.

After passing the data through some high-resolution transformations, we downsample it using {i-RevNet} downsamplers~\citep{jacobsen2018irevnet}.
Specifically, the high-resolution input is downsampled so that each one of them constitutes a different channel of the low-resolution data.

To help the normalizing flow model learn useful features, we use a fixed $1\times 1$ convolution at the beginning of each low-resolution layer.
This adjustment is done with the same spirit as in Glow~\cite{kingma2018glow}.
However, instead of having a trainable $1\times 1$ convolution, here we initialize them at the beginning of the training and keep them fixed afterward.

Finally, we used a multi-scale structure~\cite{dinh2016density} to reduce the computational complexity of the flow-based model.
Specifically, we pass the input through several layers of invertible transformations constructed using convolutional neural networks as ${\mathbf{s}_{1,2}(\cdot)}$ and ${\mathbf{t}_{1,2}(\cdot)}$.  
Then, we send three-quarters of the data directly to the ultimate output.
The rest goes through other rounds of mappings, which use fully-connected networks.
This way, one can reduce the computational burden of flow-based models as they keep the data dimension fixed.

For training, we use an Adam~\citep{kingma2015adam} optimizer with weight decay $10^{-5}$.
Besides, we set the learning rate according to an exponential scheduler starting from $10^{-4}$ and ending to $10^{-6}$.
Also, to dequantize the image pixels, we add a small Gaussian noise with ${\sigma = 0.02}$ to the pictures.
Table~\ref{tab:ap:nf_details} summarizes the hyperparameters and architecture of the flow-based model used in AdvFlow.

\begin{table}[htp]
	\caption{Hyperparameter and architecture details for normalizing flow part of AdvFlow.}
	\label{tab:ap:nf_details}
	
	\begin{center}
		\begin{small}
				\begin{tabular}{lc}
					\toprule
					Optimizer                          & Adam\\
					Scheduler                          & Exponential\\
					Initial lr                         & $10^{-4}$\\
					Final lr                           & $10^{-6}$\\
					Batch Size                         & $64$\\
					Epochs                             & $350$\\
					\midrule
					Added Noise Std.                   & $0.02$\\
					\midrule
					Multi-scale Levels                 & $2$\\
					Each Level Network Type            & CNN-FC\\
					High-res Transformation Blocks     & $4$\\
					Low-res Transformation Blocks      & $6$\\
					FC Transformation Blocks           & $6$\\
					$\alpha$ (clamping hyperparameter)& $1.5$\\
					CNN Layers Hidden Channels         & $128$\\
					FC Layers Internal Width           & $128$\\
					Activation Function                & Leaky ReLU\\
					Leaky Slope                        & $0.1$\\
					\bottomrule
				\end{tabular}
		\end{small}
	\end{center}

\end{table}

\subsection{Adversarial Example Detectors}\label{sec:ap:mah_det}

In this section, we provide the details of the LID~\citep{ma2018characterizing}, Mahalanobis~\cite{lee2018simple}, and ResFlow~\cite{zisselman2020resflow} adversarial example detectors.
All of the methods use logistic regression as their classifier, and the way that they construct their training and evaluation sets is the same. 
The only difference among these methods is the way each one extracts their features, which we review below.
The training set used to train the logistic regression classifier consists of three types of data: clean, noisy, and adversarial.
We take the test portion of each target dataset and add a slight noise to them to make the noisy data.
The clean and noisy data are going to be used as the positive samples of the logistic regression.
For the adversarial part of the data, we then use a nominated adversarial attack method and generate adversarial examples that are later used as negative samples of the logistic regression.
After constructing the entire dataset, $10\%$ of it is used as the logistic regression training set, and the rest for evaluation.
Also, the hyperparameters of the detectors are set using nested cross-validation.

\paragraph{LID Detectors.}
\citet{ma2018characterizing} use the concept of \textit{Local Intrinsic Dimensionality}~(LID) to characterize adversarial subspaces.
It is argued that for a data point that resides on some high-dimensional submanifold, its adversarially generated sample is likely to lie outside this submanifold.
As such, \citet{ma2018characterizing} argue that the intrinsic dimensionality of the adversarial examples in a local neighborhood is going to be higher than the clean or noisy data (see Figure~1 of~\cite{ma2018characterizing}).
Thus, LID can be a good measure for differentiating adversarial examples from clean data.
\citet{ma2018characterizing} then estimate the LID measure for mini-batches of data using the extreme value theory.
To this end, they extract features of the input images using a DNN classifier.
They then compute the LID score for all these features across the training and evaluation sets.
After extracting these scores for all the data, they train and evaluate the logistic regression classifier as described above.
Here, we use the PyTorch implementation of LID detectors given by \citet{lee2018simple}. 
 
\paragraph{Mahalanobis Detectors.}
\citet{lee2018simple} propose an adversarial example detector based on a Mahalanobis distance-based confidence score.
To this end, the authors extract features from different hidden layers of a nominated DNN classifier.
Assuming that these features are distributed according to class-conditional Gaussian densities, the detector aims at estimating the mean and covariance matrix associated with each one of the features across the training set.
These densities are then used to train the logistic regression classifier based on the Mahalanobis distance confidence score between a given image feature and its closest distribution.
In this paper, we use the official implementation of the Mahalanobis adversarial example detector available online.\footnote{\href{https://github.com/pokaxpoka/deep_Mahalanobis_detector}{github.com/pokaxpoka/deep\_Mahalanobis\_detector}}
For more information about these detectors, see~\citep{lee2018simple}.

\paragraph{ResFlow Detectors.}
\citet{zisselman2020resflow} generalize the Mahalanobis detectors~\citep{lee2018simple} using normalizing flows.
It is first argued that modeling the activation distributions as Gaussian densities may not be accurate.
To find a better non-Gaussian distribution, \citet{zisselman2020resflow} exploit flow-based models to construct an architecture they call Residual Flow (ResFlow).
The same procedure as in the Mahalanobis detectors is then utilized to extract features that are later used to train the logistic regression detectors.
We use the official PyTorch implementation of ResFlow available online.\footnote{\href{https://github.com/EvZissel/Residual-Flow}{github.com/EvZissel/Residual-Flow}}
Note that in the original paper, ResFlows are only used for out-of-distribution detection.
For our purposes, we generalized their implementation to adversarial example detection using the Mahalanobis detector implementation.

\subsection{Defense Methods}\label{sec:ap:defense}

In this section, we briefly review the defense techniques used in our experiments.
We will then present the set of parameters used in the training of each one of the classifiers.
Note that we only utilized these methods for the sake of evaluating our attack models, and our results cannot be regarded as a close case-study, or comparison, of the nominated defense methods.

\subsubsection{Review of Defense Methods}

\paragraph{Adversarial Training.}
Adversarial training~\citep{madry2018towards} is a method to train robust classifiers.
To achieve robustness, this method tries to incorporate adversarial examples into the training process.  
In particular, adversarial training aims at minimizing the following objective function for classifier $\mathcal{C}(\cdot)$ with parameters $\boldsymbol{\theta}$:
\begin{equation}\label{eq:ap:adv_train}
	\min_{\boldsymbol{\theta}}\sum_{i}\max_{\norm{\boldsymbol{\delta}} \leq \epsilon}\ell\big(\mathcal{C}(\mathbf{x}_{i} + \boldsymbol{\delta}), y_{i}\big).
\end{equation}
Here, $\mathbf{x}_i$ and $y_i$ are the training examples and their associated correct labels.
Also, $\ell(\cdot)$ is an appropriate cost function for classifiers, such as the standard cross-entropy loss.
The inner maximization objective in Eq.~\eqref{eq:ap:adv_train} is the cost function used to generate adversarial examples.
Thus, we can interpret Eq.~\eqref{eq:ap:adv_train} as training a model that can predict the labels correctly, even in the presence of additive perturbations.
However, finding the exact solution to the inner optimization problem is not straightforward, and in most real-world cases cannot be done efficiently~\citep{madry2018adversarial}.
To circumvent this problem, \citet{madry2018towards} proposed approximately solving it by using a Projected Gradient Descent algorithm.
This method is widely known as adversarial training.

\paragraph{Adversarial Training for Free.}
The main disadvantage of adversarial training, as proposed in~\citep{madry2018towards}, is that solving the inner optimization problem makes the algorithm much slower than standard classifier training.
This problem arises because solving the inner maximization objective requires back-propagating through the DNN.  
To address this issue, \citet{shafahi2019free} exploits a Fast Gradient Sign Method (FGSM)~\cite{goodfellow2014explaining} with step-size $\epsilon$ to compute an approximate solution to the inner maximization objective and then update the DNN parameters.
This procedure is repeated $m$ times on the same minibatch.
Finally, the total number of epochs is divided by a factor of $m$ to account for repeated minibatch training.
We use the PyTorch code available on the official repository of the \textit{free adversarial training} to train our classifiers.\footnote{\href{https://github.com/mahyarnajibi/FreeAdversarialTraining}{github.com/mahyarnajibi/FreeAdversarialTraining}} 

\paragraph{Fast Adversarial Training.}
To make adversarial training even faster, \citet{wong2020fast} came up with a method called ``fast" adversarial training.
In this approach, they combine FGSM adversarial training with the idea of random initialization to train robust DNN classifiers.
To make the proposed algorithm even faster, \citet{wong2020fast} also utilize several fast training techniques (such as cyclic learning rate and mixed-precision arithmetic) from DAWNBench competition~\cite{coleman2017dawnbench}.
In this paper, we replace the FGSM adversarial training with PGD.
However, we still use the cyclic learning rate and mixed-precision arithmetic.
For this method, we used the official PyTorch code available online.\footnote{\href{https://github.com/anonymous-sushi-armadillo}{github.com/anonymous-sushi-armadillo}}
\paragraph{Adversarial Training with Auxiliary Rotations.}

\citet{gidaris2018unsupervised} showed that Convolutional Neural Networks can learn useful image features in an unsupervised fashion by predicting the amount of rotation applied to a given image.
Throughout their experiments, they observed that these features can improve classification performance.
Motivated by these observations, \citet{hendrycks2019using} suggest exploiting the idea of self-supervised feature learning to improve the robustness of classifiers against adversaries.
Specifically, it is proposed to train a so-called ``head'' alongside the original classifier.
This auxiliary head takes the penultimate features of the classifier and aims at predicting the amount of rotation applied to an image from four possible angles ($0^{\circ}$, $90^{\circ}$, $180^{\circ}$, or $270^{\circ}$).
It was shown that this simple addition can improve the performance of adversarially trained classifiers.
To train our models, we make use of the PyTorch code for \textit{adversarial learning with auxiliary rotations} available online.\footnote{\href{https://github.com/hendrycks/ss-ood}{github.com/hendrycks/ss-ood}}

\subsubsection{Hyperparameters of Defense Methods}

Tables~\ref{tab:ap:defense_details_cifar10} and~\ref{tab:ap:defense_details_svhn} summarize the hyperparameters used for training our defended classifiers.

\begin{table}[htp]
	\caption{Hyperparameters of defense methods for training CIFAR-10~\citep{krizhevsky2009learning} classifiers.
		Numbers 1 and 2 correspond to Wide-ResNet-32~\citep{zagoruyko2016wresnet} and ResNet-50~\citep{he2016deep} architectures, respectively.}
	\label{tab:ap:defense_details_cifar10}
	
	\begin{center}
		\begin{small}
				\begin{tabular}{lcccccc}
					\toprule
					Classifier                         & Free-1         & Free-2       & Fast-1        & Fast-2        & RotNet-1       & RotNet-2\\ 
					\midrule
					Optimizer                          & SGD            & SGD          & SGD           & SGD           & SGD            & SGD\\
					lr                                 & $0.1$          & $0.005$      & $0.21$        & $0.21$        & $0.1$          & $0.1$\\
					Momentum                           & $0.9$          & $0.9$        & $0.9$         & $0.9$         & $0.9$          & $0.9$\\
					Weight Decay                       & $0.0002$       & $0.0002$     & $0.0005$      & $0.0005$      & $0.0005$       & $0.0005$\\
					Nesterov                           & N              & N            & N             & N             & Y              & Y\\
					Batch Size                         & $128$          & $64$         & $64$          & $128$         & $128$          & $128$\\
					Epochs                             & $125$          & $100$        & $100$         & $100$         & $100$          & $100$\\
					\midrule
					Inner Optimization                 & FGSM           & FGSM         & PGD           & PGD           & PGD            & PGD\\
					$\epsilon$                         & $8/255$        & $8/255$      & $8/255$       & $8/255$       & $8/255$        & $8/255$\\
					Step Size                          & $8/255$        & $8/255$      & $2/255$       & $2/255$       & $2/255$        & $2/255$\\
					Number of Steps (Repeats)          & $8$            & $8$          & $5$           & $5$           & $10$           & $10$\\
					\bottomrule
				\end{tabular}
		\end{small}
	\end{center}
	
\end{table}

\begin{table}[htp]
	\caption{Hyperparameters of defense methods for training SVHN~\citep{netzer2011reading} classifiers.
		Numbers 1 and 2 correspond to Wide-ResNet-32~\citep{zagoruyko2016wresnet} and ResNet-50~\citep{he2016deep} architectures, respectively.}
	\label{tab:ap:defense_details_svhn}

	\begin{center}
		\begin{small}
				\begin{tabular}{lcccccc}
					\toprule
					Classifier                         & Free-1         & Free-2       & Fast-1        & Fast-2        & RotNet-1       & RotNet-2\\ 
					\midrule
					Optimizer                          & SGD            & SGD          & SGD           & SGD           & SGD            & SGD\\
					lr                                 & $0.0001$       & $0.01$       & $0.21$        & $0.21$        & $0.1$          & $0.1$\\
					Momentum                           & $0.9$          & $0.9$        & $0.9$         & $0.9$         & $0.9$          & $0.9$\\
					Weight Decay                       & $0.0002$       & $0.0002$     & $0.0005$      & $0.0005$      & $0.0005$       & $0.0005$\\
					Nesterov                           & N              & N            & N             & N             & Y              & Y\\
					Batch Size                         & $128$          & $128$        & $64$          & $128$         & $128$          & $128$\\
					Epochs                             & $100$          & $100$        & $100$         & $100$         & $100$          & $100$\\
					\midrule
					Inner Optimization                 & FGSM           & FGSM         & PGD           & PGD           & PGD            & PGD\\
					$\epsilon$                         & $8/255$        & $8/255$      & $8/255$       & $8/255$       & $8/255$        & $8/255$\\
					Step Size                          & $8/255$        & $8/255$      & $2/255$       & $2/255$       & $2/255$        & $2/255$\\
					Number of Steps (Repeats)          & $8$            & $8$          & $5$           & $5$           & $10$           & $10$\\
					\bottomrule
				\end{tabular}
		\end{small}
	\end{center}

\end{table}

\subsection{Hyperparameters of Attack Methods}\label{sec:ap:attack_hyper}

In this part, we present the set of hyperparameters used for each attack method.
For $\mathcal{N}$\textsc{Attack}~\citep{li2019nattack} and AdvFlow, we tune the hyperparameters on a development set so that they result in the best performance for an un-defended CIFAR-10 classifier.
In the case of bandits with time and data-dependent priors~\cite{ilyas2019prior}, we use two sets of hyperparameters tuned for these methods.
For the vanilla classifiers we use the hyperparameters set in~\citep{ilyas2019prior}, while for defended classifiers we use those set in~\citep{moon2019parsimonous}.
For SimBA~\citep{guo2019simba}, we used the hyperparameters set in the official repository. 
We only changed the stride from $7$ to $6$ to allow for the correct computation of block reordering. 
Once set, we keep the hyperparameters fixed throughout the rest of experiments.
Tables~\ref{tab:ap:bandits_details}-\ref{tab:ap:AdvFlow_details} summarize the hyperparameters used for each attack method in our experiments.

\begin{table}[H]
	\caption{Hyperparameters of bandits with time and data-dependent priors~\cite{ilyas2019prior}.}
	\label{tab:ap:bandits_details}

	\begin{center}
		\begin{small}
				\begin{tabular}{lcc}
					\toprule
					Hyperparameter                     & Vanilla              & Defended\\
					\midrule
					OCO learning rate                  & $100$                & $0.1$\\
					Image learning rate                & $0.01$               & $0.01$\\
					Bandit exploration                 & $0.1$                & $0.1$\\
					Finite difference probe            & $0.1$                & $0.1$\\
					Tile size                          & $(6\mathrm{px})^{2}$ & $(4\mathrm{px})^{2}$\\
					\bottomrule
				\end{tabular}
		\end{small}
	\end{center}

\end{table}
\begin{table}[H]
	\caption{Hyperparameters of $\mathcal{N}$\textsc{Attack}~\citep{li2019nattack}.}
	\label{tab:ap:NATTACK_details}

	\begin{center}
		\begin{small}
				\begin{tabular}{lc}
					\toprule
					Hyperparameter          & Value\\
					\midrule
					$\sigma$ (noise std.)   & $0.1$\\
					Sample size             & $20$\\
					Learning rate           & $0.02$\\
					Maximum iteration       & $500$\\
					\bottomrule
				\end{tabular}
		\end{small}
	\end{center}

\end{table}
\begin{table}[H]
	\caption{Hyperparameters of SimBA~\citep{guo2019simba}.}
	\label{tab:ap:SimBA_details}

	\begin{center}
		\begin{small}
			\begin{tabular}{lc}
				\toprule
				Hyperparameter          & Value\\
				\midrule
				$\epsilon$              & $0.2$\\
				Freq. Dimensionality    & $14$\\
				Order                   & Strided\\
				Stride                  & $6$\\
				\bottomrule
			\end{tabular}
		\end{small}
	\end{center}

\end{table}
\begin{table}[H]
	\caption{Hyperparameters of AdvFlow (ours).}
	\label{tab:ap:AdvFlow_details}

	\begin{center}
		\begin{small}
				\begin{tabular}{lc}
					\toprule
					Hyperparameter          & Value\\
					\midrule
					$\sigma$ (noise std.)   & $0.1$\\
					Sample size             & $20$\\
					Learning rate           & $0.02$\\
					Maximum iteration       & $500$\\
					\bottomrule
				\end{tabular}
		\end{small}
	\end{center}

\end{table}

\clearpage
\section{Extended Experimental Results}\label{sec:ap:ext_sim_res}

In this section, we present an extended version of our experimental results.

\subsection{Table of Attack Success Rate and Number of Queries}

Table~\ref{tab:ap:success} presents attack success rate, as well as average and median of the number of queries for AdvFlow alongside bandits~\cite{ilyas2019prior} and $\mathcal{N}$\textsc{Attack}~\cite{li2019nattack}.
In each case, we have also shown the clean data accuracy and success rate of the white-box PGD-100 attack for reference.
Details of classifier training and defense mechanism can be found in Appendix~\ref{sec:ap:defense}.
As can be seen, when it comes to attacking defended models, AdvFlow can outperform the baselines in both the number of queries and attack success rate. 

\subsection{Success Rate vs. Number of Queries}

Figure~\ref{fig:ap_succ_query} shows the success rate of AdvFlow and $\mathcal{N}$\textsc{Attack}~\citep{li2019nattack} as a function of the maximum number of queries for defended models.
As can be seen, given a fixed number of queries, AdvFlow can generate more successful attacks.

\subsection{Confusion Matrices of Transferability}

Figure~\ref{fig:ap_confusion_trans} shows the transferability rate of generated attacks to various classifiers.
Each entry shows the success rate of adversarial examples intended to attack the row-wise classifier in attacking the column-wise classifier.
There are a few points worth mentioning regarding these results:
\begin{itemize}
	\item AdvFlow attacks are more transferable between defended models than vanilla to defended models.
	We argue that the underlying reason is the fact that AdvFlow learns a higher-level perturbation to attack DNNs.
	As a result, since vanilla classifiers use different features than the defended ones, they are less adaptable to attack defended classifiers.
	In contrast, since ${\mathcal{N}}$\textsc{Attack} acts on a pixel level, they are less susceptible to this issue.
	\item  Generally, adversarial examples generated by AdvFlow are more transferable between different architectures than ${\mathcal{N}}$\textsc{Attack}.
	The same argument as in our previous point applies here.
	\item Transferability of black-box attacks is not as important as in the white-box setting.
	The reason is that in the case of black-box attacks, since no assumption is made about the model architecture, we can try to generate new adversarial examples to attack a new target classifier.
	However, for white-box attacks, transferability is somehow related to their success in attacking previously unseen target networks.
	Thus, it is essential to have a high rate of transferability if a white-box attack is meant to be deployed in real-world situations where often, we do not have any access to internal nodes of a classifier.
\end{itemize}

\subsection{Samples of Adversarial Examples}

Figure~\ref{fig:samples} shows samples of adversarial examples generated by AdvFlow and $\mathcal{N}$\textsc{Attack}~\citep{li2019nattack}, intended to attack a vanilla Wide-ResNet-32~\citep{zagoruyko2016wresnet}.
As the images show, AdvFlow can generate adversarial perturbations that often take the shape of the original data.
This property makes AdvFlows less detectable to adversarial example detectors.
In contrast, it is clear that the perturbations generated by $\mathcal{N}$\textsc{Attack} come from a different distribution than the data itself.
As a result, they can be detected easily by adversarial example detectors.  

\clearpage

\begin{sidewaystable}[tb!]
	\caption{Attack success rate, average and median of the number of queries to generate an adversarial example for CIFAR-10~\citep{krizhevsky2009learning} and SVHN~\citep{netzer2011reading}.
		For a fair comparison, we first find the samples where all the attack methods are successful, and then compute the average (median) of queries for these samples.
		Note that for $\mathcal{N}$\textsc{Attack} and AdvFlow we check whether we arrived at an adversarial point every $200$ queries, and hence, the medians are multiples of $200$.
		Clean data accuracy and PGD-100 attack success rate are also shown for reference.
	    All attacks are with respect to $\ell_{\infty}$ norm with $\epsilon_{\max}=8/255$.}
	\label{tab:ap:success}
	\begin{center}
		\begin{small}
				\begin{tabular}{cccccccc}
					\toprule
					\parbox[t]{2mm}{\multirow{2}{*}{\rotatebox[origin=c]{90}{Arch}}} & \parbox[t]{2mm}{\multirow{2}{*}{\rotatebox[origin=c]{90}{Data}}} &&Attack & PGD-100 & \multicolumn{3}{c}{Bandits~\cite{ilyas2019prior} / $\mathcal{N}$\textsc{Attack}~\cite{li2019nattack} / SimBA~\cite{guo2019simba} / AdvFlow (ours)}\\
					\cmidrule(lr){6-8}
					&&Defense & Clean Acc.(\%) & {Success Rate(\%) \textuparrow} & \multicolumn{1}{c}{Success Rate(\%) \textuparrow}   & \multicolumn{1}{c}{Avg. of Queries \textdownarrow} & \multicolumn{1}{c}{Med. of Queries \textdownarrow}\\
					\cmidrule(lr){1-8}
					\parbox[t]{2mm}{\multirow{8}{*}{\rotatebox[origin=c]{90}{WideResNet32~\citep{zagoruyko2016wresnet}}}}
					&\parbox[t]{2mm}{\multirow{4}{*}{\rotatebox[origin=c]{90}{CIFAR-10}}}
					&Vanilla                               & $91.77$    & $100$   & $98.81$ / $\mathbf{100}$ / $99.99$ / $99.42$    & $552.69$	/ $\mathbf{237.58}$ / $237.70$ / $949.31$ & $182$ / $200$ / $\mathbf{126}$ / $400$\\
					&&FreeAdv~\citep{shafahi2019free}      & $81.29$    & $47.52$ & $37.12$ / $38.97$ / $35.52$ / $\mathbf{41.21}$  & $1062.70$	/ $874.91$ / $463.09$ / $\mathbf{421.63}$ & $354$ / $400$ / $244$ / $\mathbf{200}$\\
					&&FastAdv~\citep{wong2020fast}         & $86.33$    & $46.37$ & $36.60$ / $36.90$ / $35.07$ / $\mathbf{40.22}$  & $1065.92$	/ $973.05$ / $\mathbf{428.81}$ / $436.80$ & $358$ / $400$ / $234$ / $\mathbf{200}$\\
					&&RotNetAdv~\citep{hendrycks2019using} & $86.58$    & $46.59$ & $37.73$ / $38.04$ / $35.63$ / $\mathbf{40.67}$  & $1085.43$	/ $941.67$ / $471.99$ / $\mathbf{424.95}$ & $408$ / $400$ / $259$ / $\mathbf{200}$\\
					\cmidrule(lr){2-8}
					&\parbox[t]{2mm}{\multirow{4}{*}{\rotatebox[origin=c]{90}{SVHN}}} 
					&Vanilla                               & $96.45$    & $99.81$ & $87.84$	/ $\mathbf{98.76}$ / $97.26$ / $90.31$ & $1750.65$ / $408.75$ /	$\mathbf{202.07}$ /	$1572.24$ & $1128$ / $200$ / $\mathbf{107}$ / $600$\\
					&&FreeAdv~\citep{shafahi2019free}      & $86.47$    & $57.22$ & $49.64$	/ $50.28$ /	$46.28$	/ $\mathbf{50.76}$ & $819.98$ / $903.12$ / $\mathbf{365.42}$ / $692.73$   & $250$ / $400$ / $216$ / $\mathbf{200}$\\
					&&FastAdv~\citep{wong2020fast}         & $93.90$    & $46.76$ & $40.43$	/ $35.42$ /	$36.19$	/ $\mathbf{41.49}$ & $755.23$ / $1243.38$ /	$\mathbf{307.73}$ / $526.37$  & $284$ / $600$ / $216$ / $\mathbf{200}$\\
					&&RotNetAdv~\citep{hendrycks2019using} & $90.33$    & $48.67$ & $43.47$	/ $41.49$ /	$39.01$	/ $\mathbf{44.22}$ & $663.07$ / $756.48$ / $\mathbf{319.93}$ / $480.02$   & $202$ / $400$ / $\mathbf{186}$ / $200$\\
					\cmidrule(lr){1-8}
					\parbox[t]{2mm}{\multirow{8}{*}{\rotatebox[origin=c]{90}{ResNet50~\citep{he2016deep}}}}
					&\parbox[t]{2mm}{\multirow{4}{*}{\rotatebox[origin=c]{90}{CIFAR-10}}}
					&Vanilla                               & $91.75$    & $100$   & $96.75$ / $99.85$ / $\mathbf{99.96}$ / $99.37$  & $795.28$ / $\mathbf{252.13}$ / $286.05$ / $1051.18$ & $280$ / $200$ / $\mathbf{163}$ / $600$\\
					&&FreeAdv~\citep{shafahi2019free}      & $75.17$    & $54.54$ & $45.64$ / $46.49$ / $43.14$ / $\mathbf{49.46}$  & $842.56$ / $836.81$ / $383.56$ / $\mathbf{371.81}$  & $248$ / $400$ / $206$ / $\mathbf{200}$\\
					&&FastAdv~\citep{wong2020fast}         & $79.09$    & $53.45$ & $45.20$ / $45.19$ / $43.57$ / $\mathbf{49.08}$  & $891.54$ / $901.44$ / $374.58$ / $\mathbf{359.21}$  & $248$ / $400$ / $\mathbf{184}$ / $200$\\
					&&RotNetAdv~\citep{hendrycks2019using} & $76.39$    & $52.04$ & $45.80$ / $46.41$ / $42.65$ / $\mathbf{50.10}$  & $826.60$ / $774.24$ / $376.30$ / $\mathbf{292.74}$  & $232$ / $400$ / $\mathbf{184}$ / $200$\\
					\cmidrule(lr){2-8}
					&\parbox[t]{2mm}{\multirow{4}{*}{\rotatebox[origin=c]{90}{SVHN}}}
					&Vanilla                               & $96.23$    & $99.38$ & $92.63$	/ $\mathbf{96.73}$ / $93.14$ / $83.67$ & $1338.30$ / $487.32$ / $\mathbf{250.02}$ / $1749.48$ & $852$ / $200$ /$\mathbf{126}$ / $800$\\
					&&FreeAdv~\citep{shafahi2019free}      & $87.67$    & $46.50$ & $42.27$	/ $43.99$ / $39.83$ / $\mathbf{44.66}$ & $793.30$ /	$703.76$ / $\mathbf{327.30}$ / $565.2$    & $\mathbf{198}$ / $400$ / $207$ / $200$\\
					&&FastAdv~\citep{wong2020fast}         & $92.67$    & $50.25$ & $43.26$	/ $36.99$ / $38.98$ / $\mathbf{45.11}$ & $739.40$ /	$1255.24$ / $\mathbf{286.71}$ / $436.83$  & $294$ / $600$ / $202$ / $\mathbf{200}$\\
					&&RotNetAdv~\citep{hendrycks2019using} & $90.15$    & $48.30$ & $43.17$	/ $40.37$ / $39.00$ / $\mathbf{43.96}$ & $660.81$ /	$891.44$ / $\mathbf{312.47}$ / $497.74$   & $\mathbf{190}$ / $400$ / $195$ / $200$\\
					\bottomrule
				\end{tabular}
		\end{small}
	\end{center}
	\vskip -0.1in
\end{sidewaystable}

\clearpage
\begin{sidewaysfigure}[tb!]
	\centering
	\begin{subfigure}{.45\textwidth}
		\centering
		\includegraphics[width=0.9\textwidth]{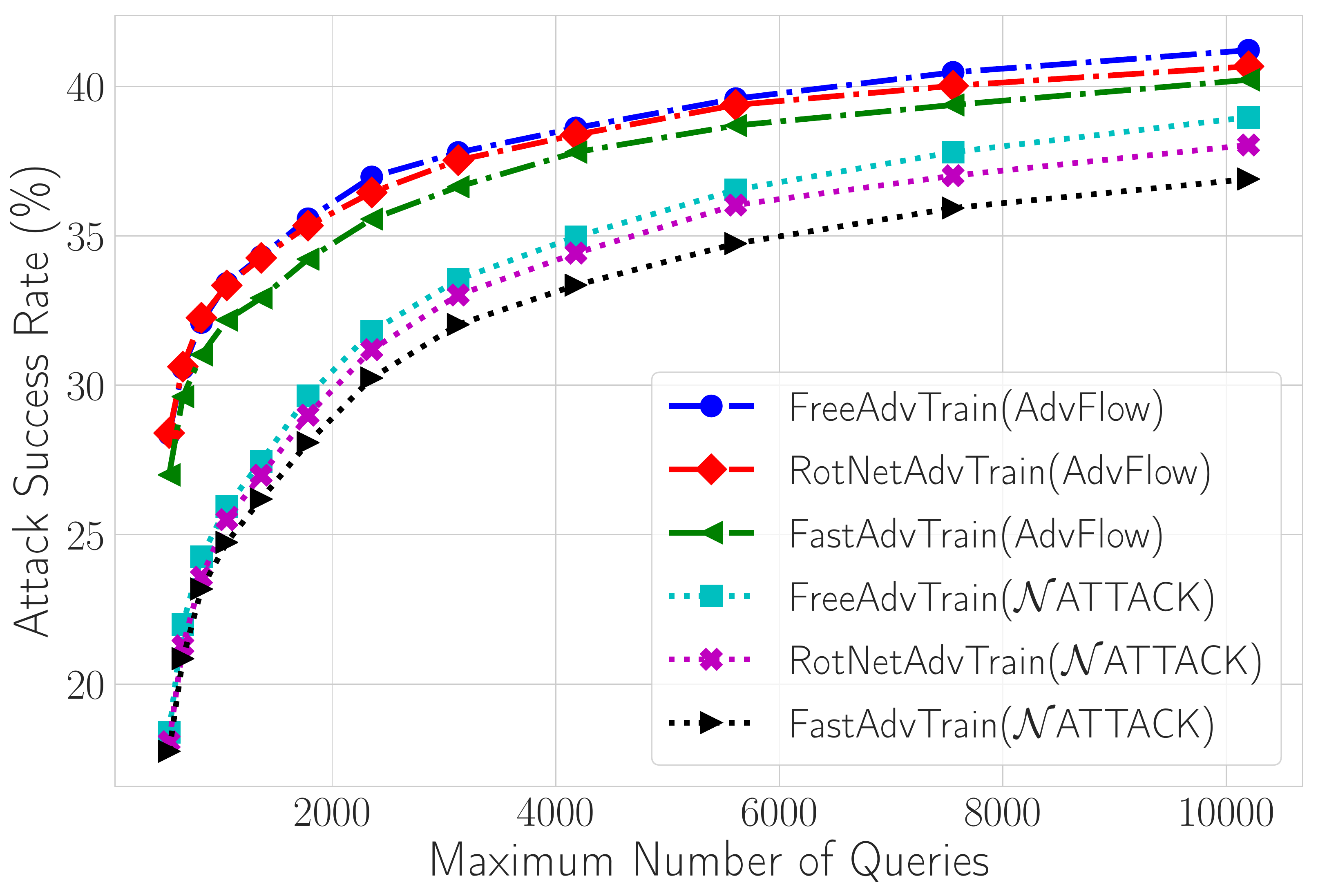}
		\caption*{{CIFAR-10 (Wide-ResNet-32)}}
	\end{subfigure}
	\begin{subfigure}{.45\textwidth}
		\centering
		\includegraphics[width=0.9\textwidth]{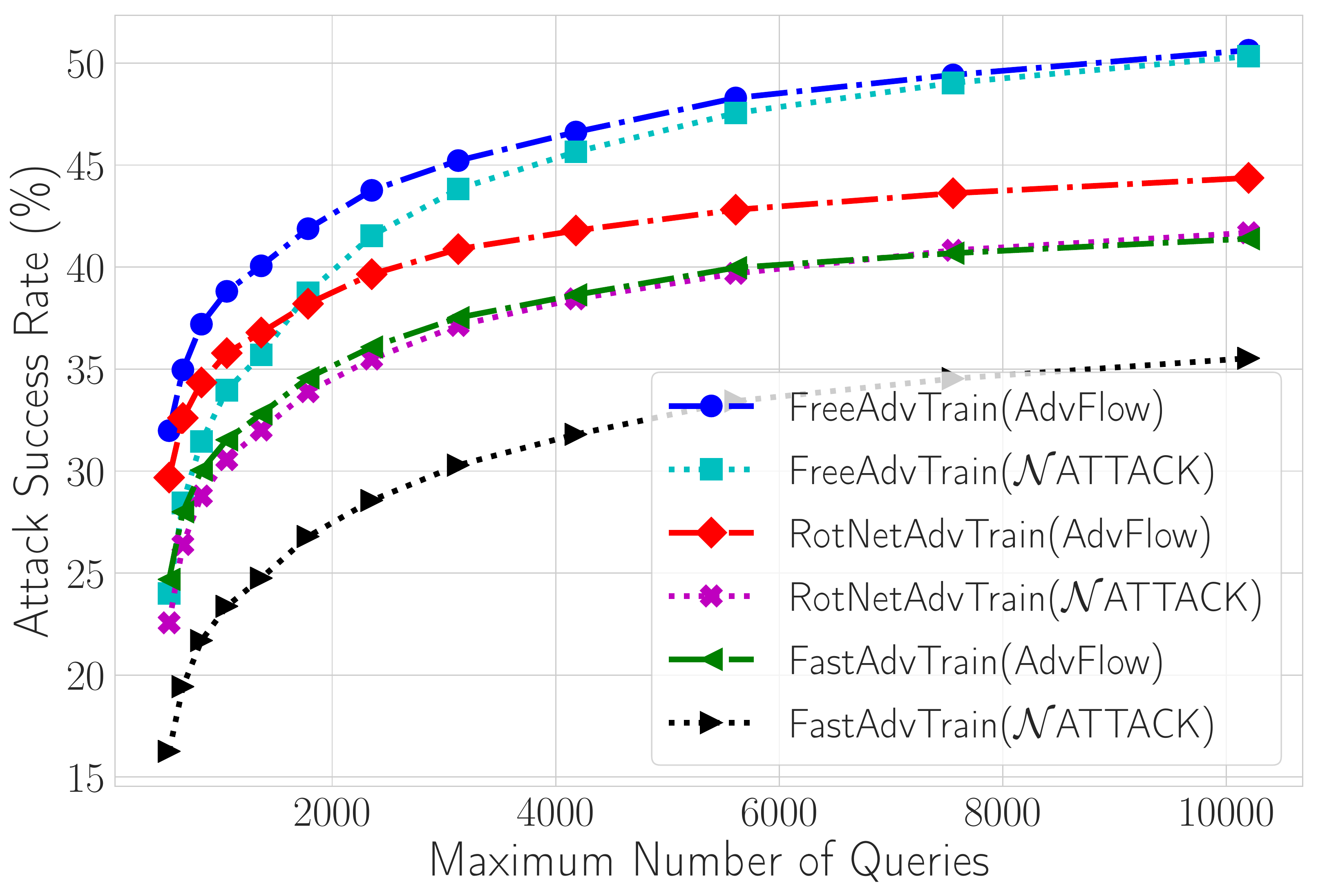}
		\caption*{{SVHN (Wide-ResNet-32)}}
	\end{subfigure}\\\vspace*{2em}
	\begin{subfigure}{.45\textwidth}
		\centering
		\includegraphics[width=0.9\textwidth]{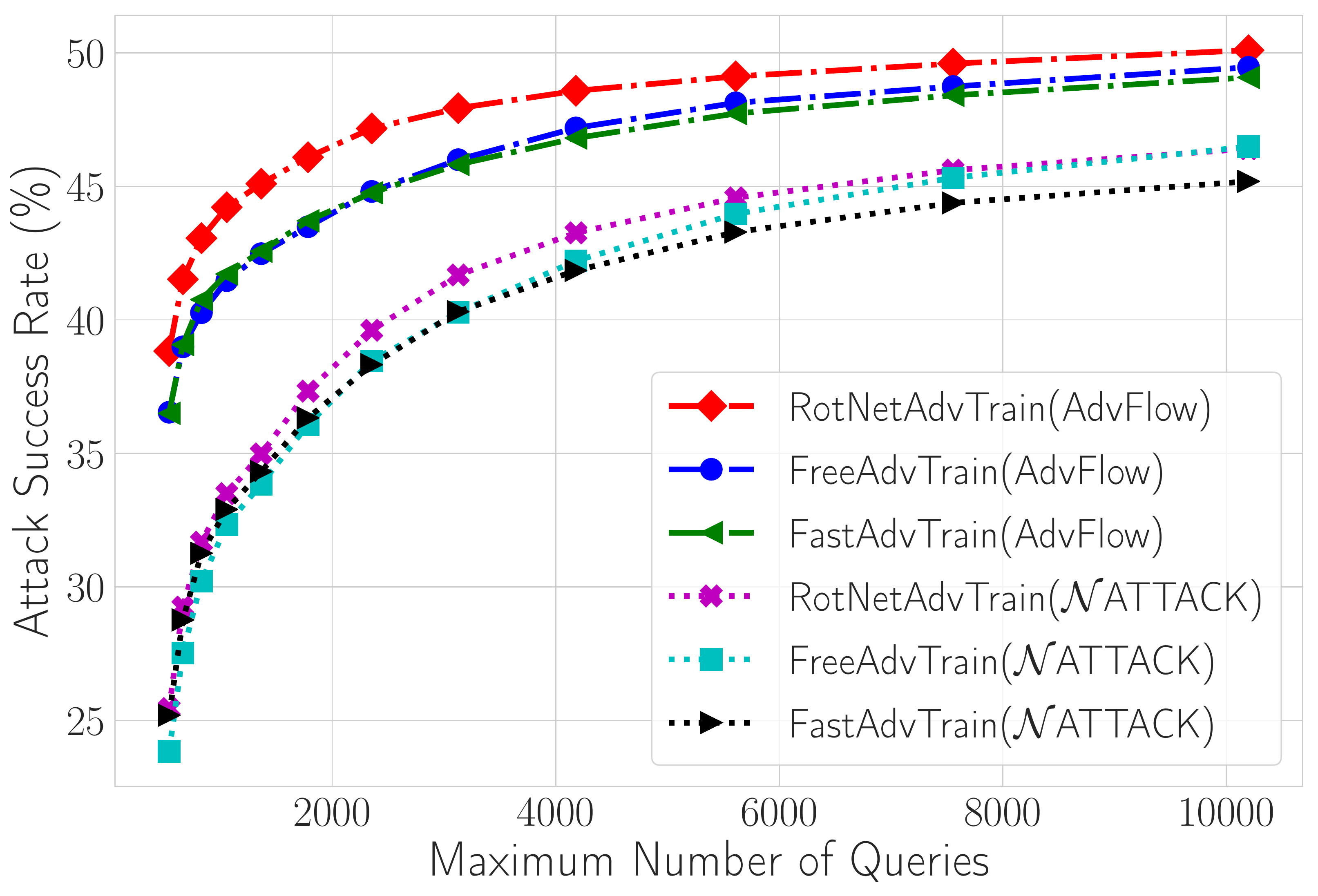}
		\caption*{{CIFAR-10 (ResNet-50)}}
	\end{subfigure}
	\begin{subfigure}{.45\textwidth}
		\centering
		\includegraphics[width=0.9\textwidth]{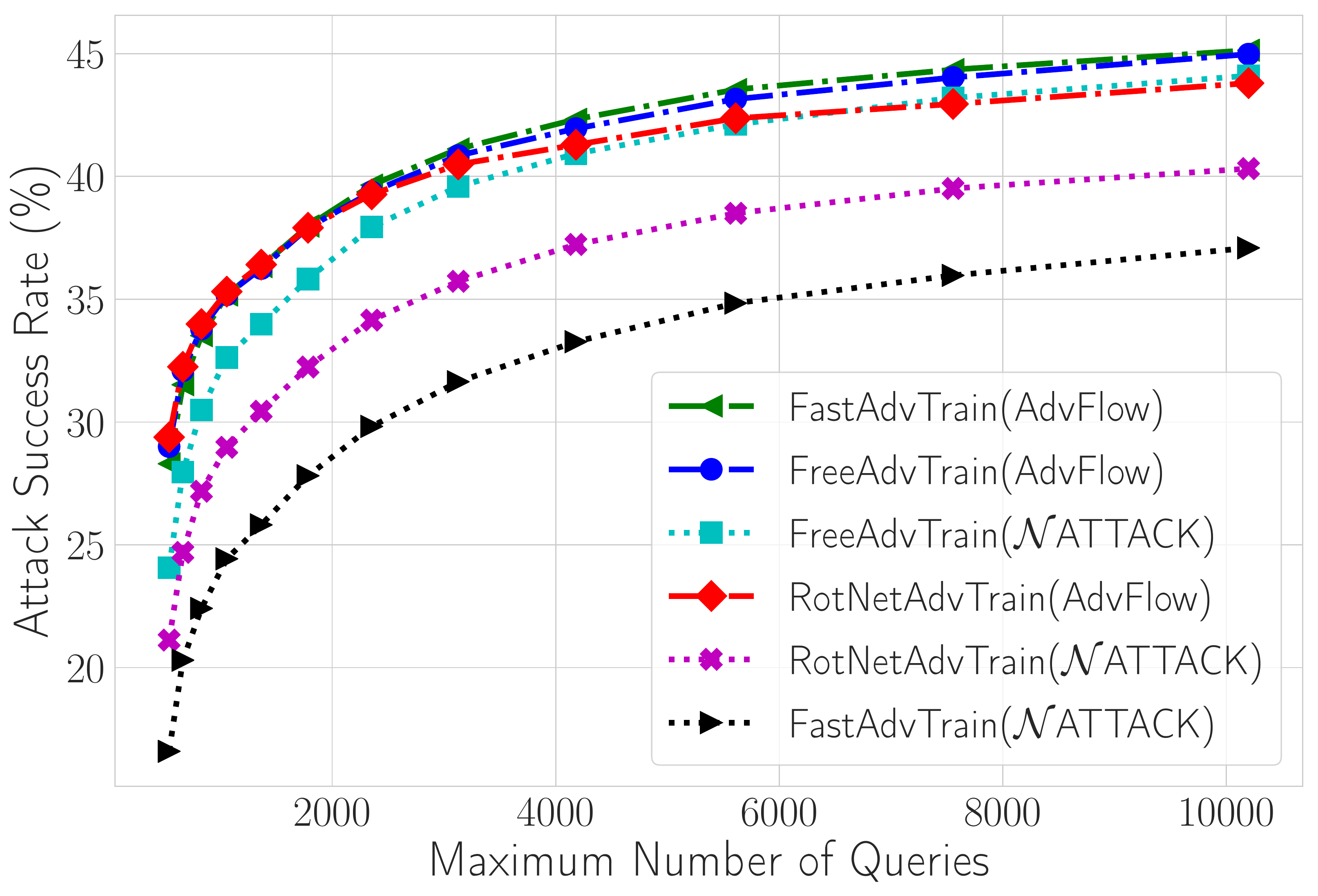}
		\caption*{{SVHN (ResNet-50)}}
	\end{subfigure}\vspace*{2em}
	\caption{Success rate vs. maximum number of queries to attack CIFAR-10~\cite{krizhevsky2009learning} and SVHN~\cite{netzer2011reading} classifiers with Wide-ResNet-32~\cite{zagoruyko2016wresnet} and ResNet-50~\cite{he2016deep} architectures.}
	\label{fig:ap_succ_query}
\end{sidewaysfigure}

\begin{figure*}[htp]
	\centering
	\begin{subfigure}{.50\textwidth}
		\centering
		\includegraphics[width=1.0\linewidth]{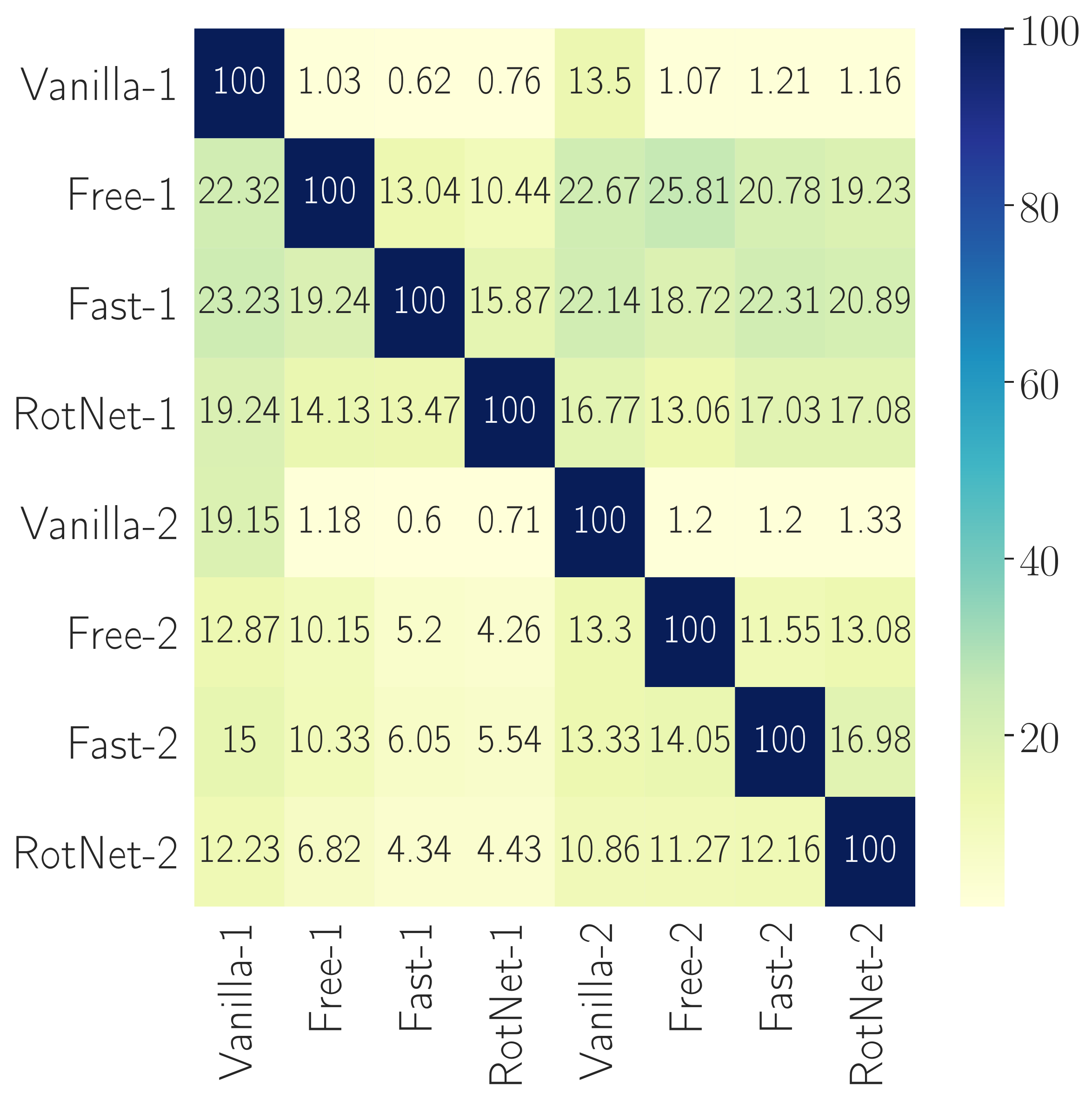}
		\caption*{CIFAR-10 (${\mathcal{N}}$\textsc{Attack})}
	\end{subfigure}\hspace*{1.5em}
	\begin{subfigure}{.50\textwidth}
		\centering
		\includegraphics[width=1.0\linewidth]{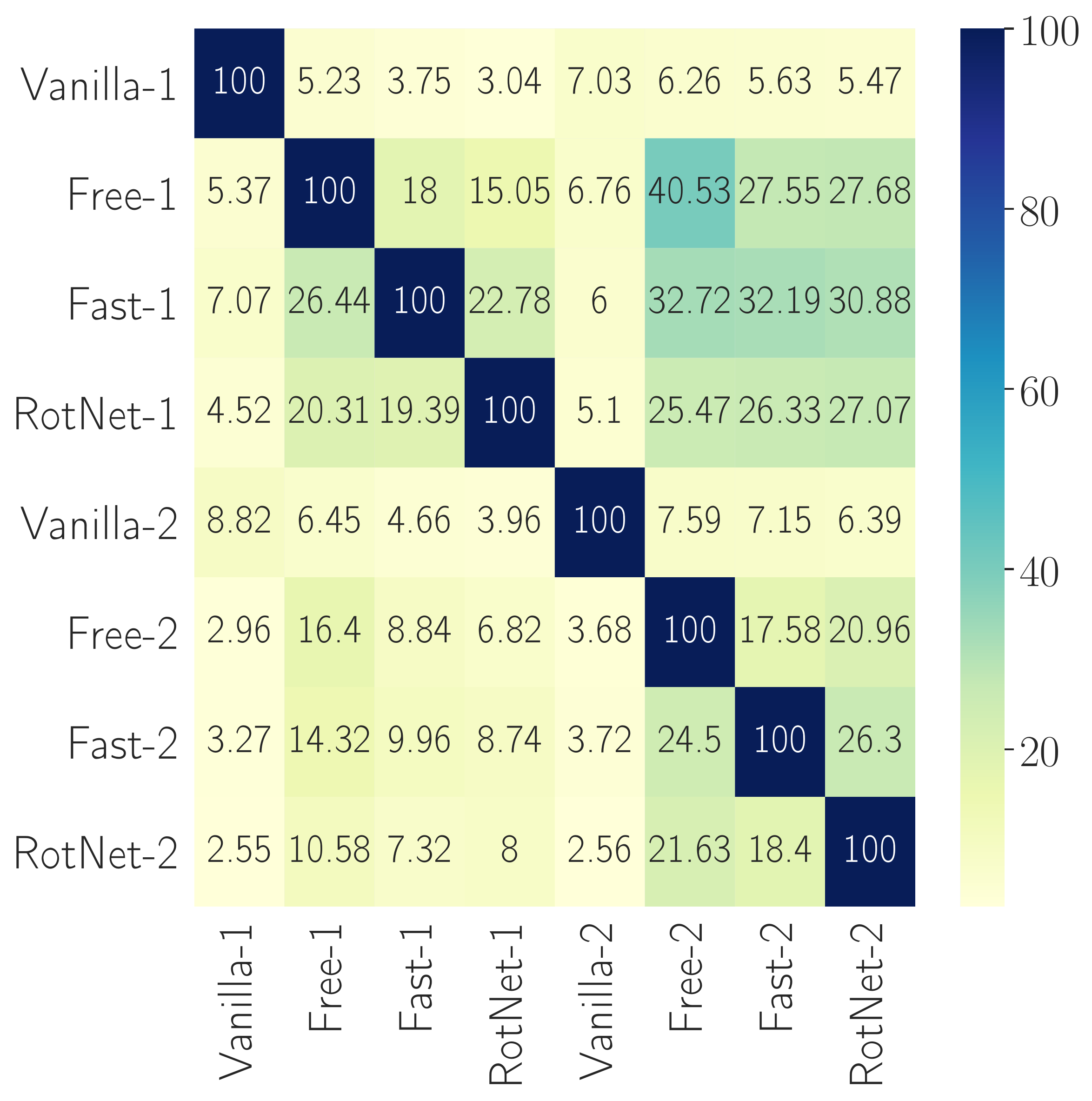}
		\caption*{CIFAR-10 (AdvFlow)}
	\end{subfigure}\\\vspace*{1em}
	\begin{subfigure}{.50\textwidth}
		\centering
		\includegraphics[width=1.0\linewidth]{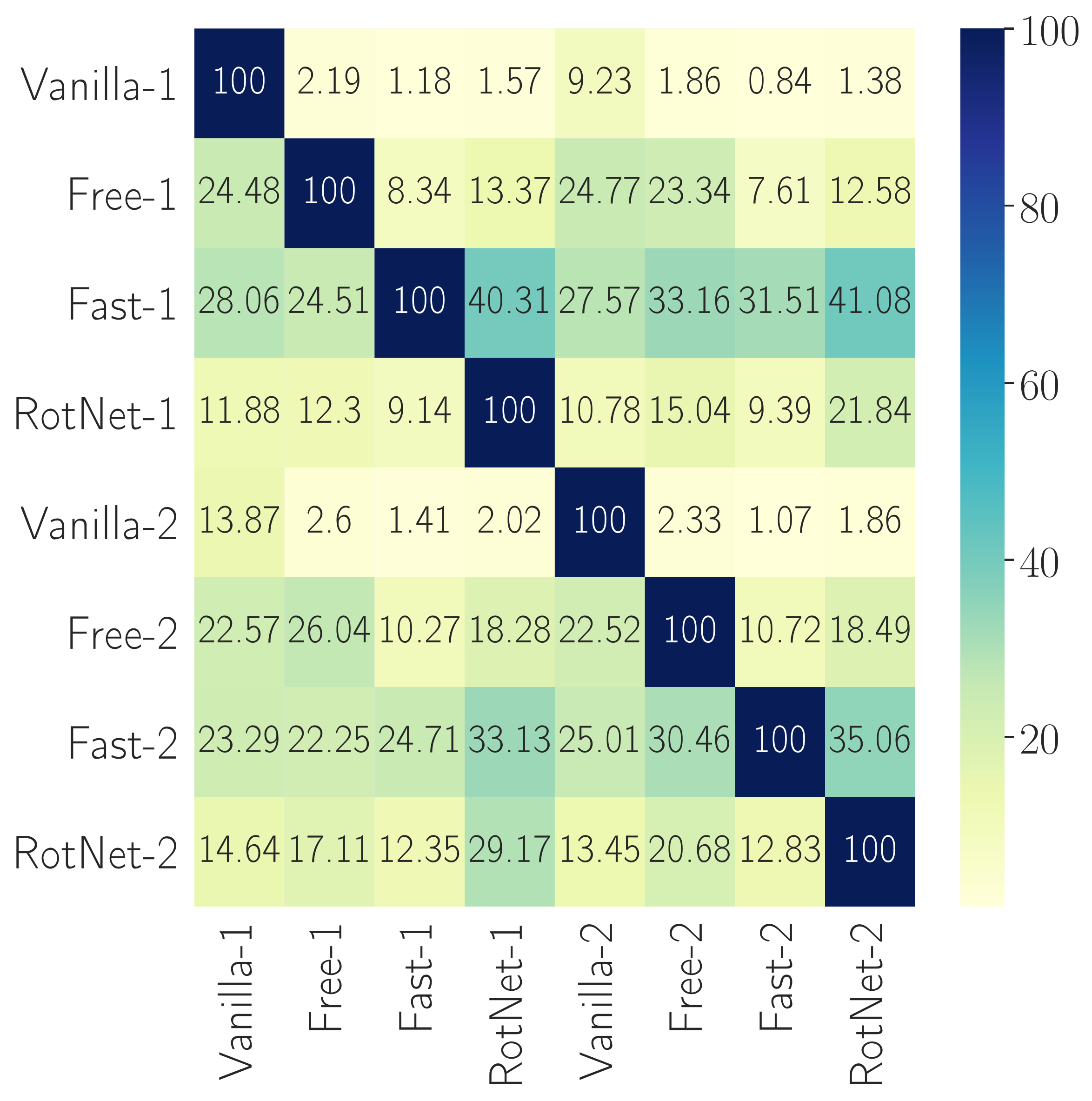}
		\caption*{{SVHN (${\mathcal{N}}$\textsc{Attack})}}
	\end{subfigure}\hspace*{1.5em}
	\begin{subfigure}{.50\textwidth}
		\centering
		\includegraphics[width=1.0\linewidth]{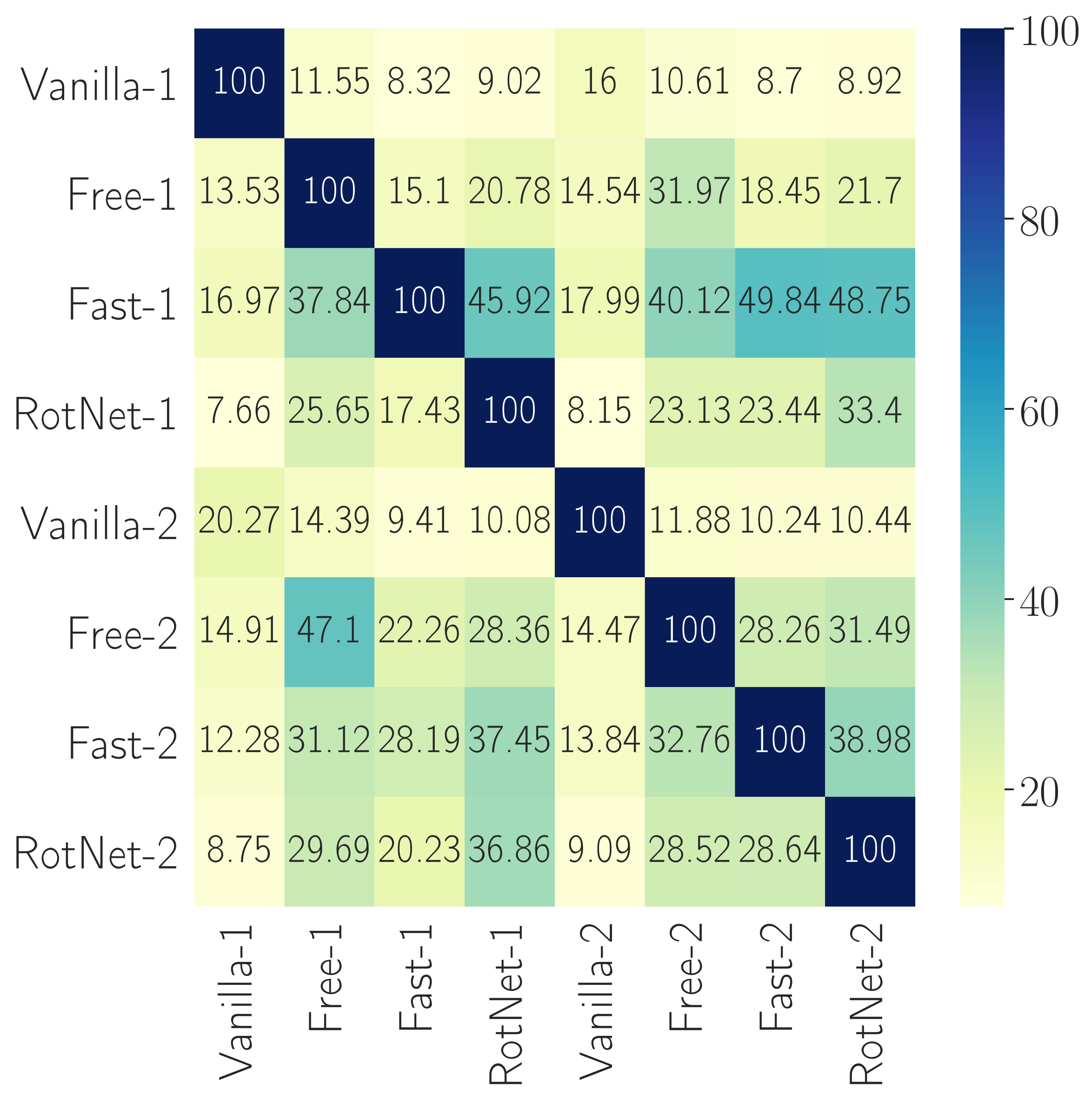}
		\caption*{{SVHN (AdvFlow)}}
	\end{subfigure}
	\caption{Confusion matrix of transferability for adversarial attacks generated by AdvFlow (ours) and ${\mathcal{N}}$\textsc{Attack}~\cite{li2019nattack}.
		Each entry shows the success rate of adversarial examples originally generated for the row-wise classifier to attack the column-wise model.
		Also, the numbers 1 and 2 in the name of each classifier indicates whether it has a Wide-ResNet-32~\citep{zagoruyko2016wresnet} or ResNet-50~\citep{he2016deep} architecture, respectively.}
	\label{fig:ap_confusion_trans}
\end{figure*}

\begin{figure}[htp]
	\centering
	\begin{subfigure}{.5\textwidth}
		\centering
		\includegraphics[width=1.0\textwidth]{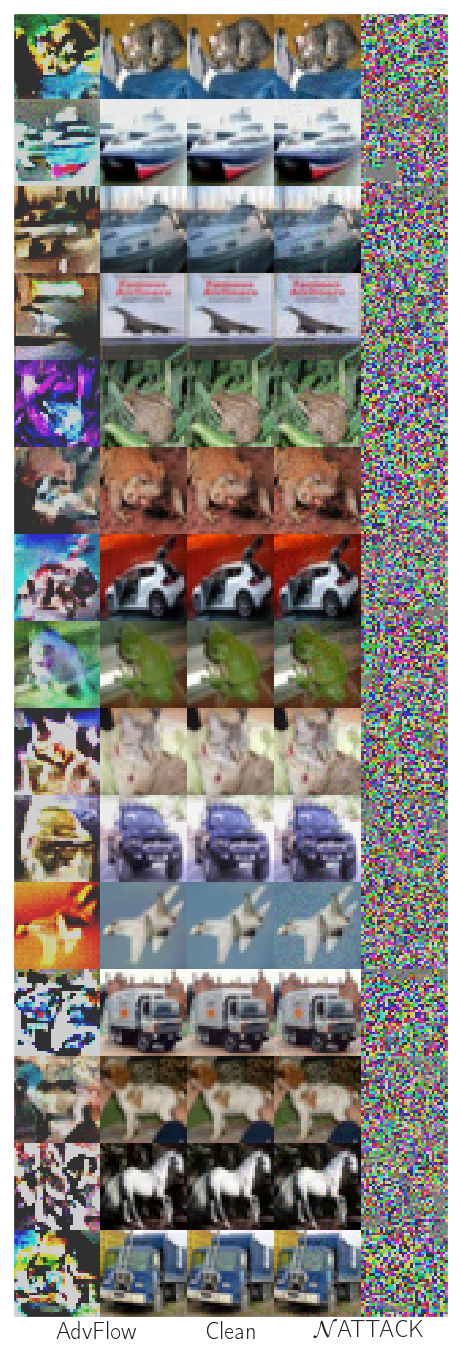}
		\caption*{CIFAR-10~\cite{krizhevsky2009learning}}
	\end{subfigure}\hspace*{0.2em}
	\begin{subfigure}{.5\textwidth}
		\centering
		\includegraphics[width=1.0\textwidth]{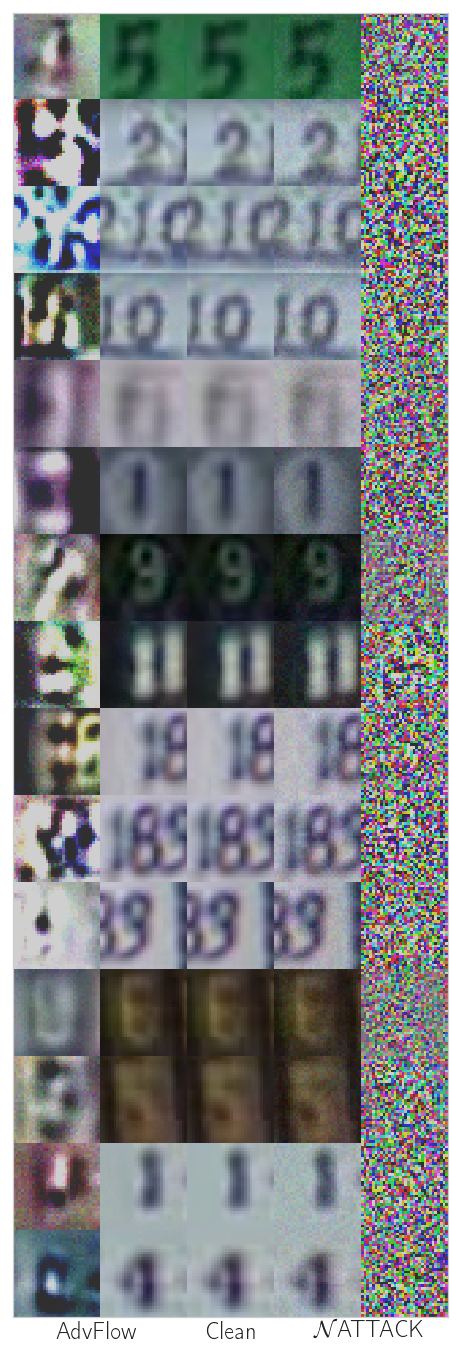}
		\caption*{SVHN~\cite{netzer2011reading}}
	\end{subfigure}
	\caption{Magnified difference and adversarial examples generated by AdvFlow (ours) and $\mathcal{N}$\textsc{Attack}~\cite{li2019nattack} alongside the clean data.
		As can be seen, the adversaries generated by AdvFlow are better disguised in the data, while $\mathcal{N}$\textsc{Attack}~\cite{li2019nattack} look noisy (better viewed in digital format).}
	\label{fig:samples}
\end{figure}
\clearpage
\subsection{Training Data and Its Effects}\label{sec:ap:training_data}

In the closing paragraph of Section~\ref{sec:sec:problem_statement} we discussed that while we are using the same training data as the classifier for our flow-based models, it does not have any effect on the performance of generated adversarial examples.
This claim is valid since we did not explicitly use this data to either generate adversarial examples or learn their features by which the classifier intends to distinguish them.
To support our claim, we replicate the experiments of Table~\ref{tab:success} for the case of CIFAR-10~\cite{krizhevsky2009learning}.
This time, however, we select $9000$ of the test data and pre-train our normalizing flow on them.
Then, we evaluate the performance of AdvFlow on the remaining $1000$ test data.
We then report the attack success rate and the average number of queries in Table~\ref{tab:ap:data_1000_9000} for this new scenario (Scenario 2) vs. the original case (Scenario 1).

As can be seen, the performance does not change in general.
The little differences between the two cases come from the fact that in Scenario 1, we had $50000$ training data to train our flow-based model, while in Scenario 2 we only trained our model on $9000$ training data.  
Also, here we are evaluating AdvFlow performance on $1000$ test data in contrast to the whole $10000$ test images used in assessing Scenario 1.
Finally, it should be noted that we get even a more balanced relative performance for the fair case where we split the training data $50/50$ between classifier and flow-based model.
However, since the performance of the classifier drops in this case, we only report the unfair situation here.

\begin{table*}[htp]
	\caption{Attack success rate and average (median) of the number of queries to generate an adversarial example for CIFAR-10~\citep{krizhevsky2009learning}.
		Scenario 1 corresponds to the case where we use the whole CIFAR-10 training data to train our normalizing flow.
		Scenario 2 indicates the experiment in which we train our flow-based model on $9000$ images from CIFAR-10 test data.  
		The architecture of the classifier in all of the cases is Wide-ResNet-32~\citep{zagoruyko2016wresnet}.
		Also, all attacks are with respect to $\ell_{\infty}$ norm with $\epsilon_{\max}=8/255$.}
	\label{tab:ap:data_1000_9000}
	\begin{center}
		\begin{small}
				\begin{tabular}{cccccc}
					\toprule
					\parbox[t]{2mm}{\multirow{2}{*}{\rotatebox[origin=c]{90}{Data}}}
					&                             & \multicolumn{2}{c}{Success Rate(\%) \textuparrow}   & \multicolumn{2}{c}{Avg. (Med.) of Queries \textdownarrow}\\
					\cmidrule(lr){3-4}\cmidrule(lr){5-6}
					&Defense                      & Scenario 1              & Scenario 2   & Scenario 1                                & Scenario 2\\
					\midrule
					\parbox[t]{2mm}{\multirow{4}{*}{\rotatebox[origin=c]{90}{CIFAR-10}}}
					&Vanilla                      & $99.42$                 & $98.91$      & $950.07$ ($400$)                          &  $949.78$ ($400$)\\
					&FreeAdv                      & $41.21$                 & $40.22$      & $923.58$ ($200$)                          &  $962.31$ ($200$)\\
					&FastAdv                      & $40.22$                 & $40.93$      & $963.77$ ($200$)                          &  $1114.68$ ($200$)\\
					&RotNetAdv                    & $40.67$                 & $40.32$      & $880.86$ ($200$)                          &  $876.57$ ($400$)\\
					\bottomrule
				\end{tabular}
		\end{small}
	\end{center}

\end{table*}

More interestingly, we observe that in case we train our flow-based model on some similar dataset to the original one, we can still get an acceptable relative performance.
More specifically, we train our flow-based model on CIFAR-100~\citep{krizhevsky2009learning} dataset instead of CIFAR-10~\citep{krizhevsky2009learning}.
Then, we perform AdvFlow on the test data of CIFAR-10~\citep{krizhevsky2009learning}.
We know that despite being visually similar, these two datasets have their differences in terms of classes and samples per class.

Table~\ref{tab:ap:cifar_100} shows the performance of AdvFlow in this case where it is pre-trained on CIFAR-100 instead of CIFAR-10.
As the results indicate, we can achieve a competitive performance despite our model being trained on a slightly different dataset.
Furthermore, a few adversarial examples from this model are shown in Figure~\ref{fig:cifar100_samples}.
We see that the perturbations are still more or less taking the shape of the data.

\begin{table*}[htp]
	\caption{Attack success rate and average (median) of the number of queries to generate an adversarial example for CIFAR-10~\citep{krizhevsky2009learning} test data.
		The \textit{train data} row shows the data that is used for training the normalizing flow part of AdvFlow. 
		The architecture of the classifier in all of the cases is Wide-ResNet-32~\citep{zagoruyko2016wresnet}.
		Also, all attacks are with respect to $\ell_{\infty}$ norm with $\epsilon_{\max}=8/255$.}
	\label{tab:ap:cifar_100}
	\begin{center}
		\begin{small}
				\begin{tabular}{cccccc}
					\toprule
					\parbox[t]{2mm}{\multirow{2}{*}{\rotatebox[origin=c]{90}{Test}}}
					&                             & \multicolumn{2}{c}{Success Rate(\%) \textuparrow}   & \multicolumn{2}{c}{Avg. (Med.) of Queries \textdownarrow}\\
					\cmidrule(lr){3-4}\cmidrule(lr){5-6}
					&Def/Train Data               & CIFAR-10                & CIFAR-100    & CIFAR-10                                  & CIFAR-100\\
					\midrule
					\parbox[t]{2mm}{\multirow{4}{*}{\rotatebox[origin=c]{90}{CIFAR-10}}}
					&Vanilla                      & $99.42$                 & $98.72$      & $950.07$ ($400$)                          &  $1198.03$ ($600$)\\
					&FreeAdv                      & $41.21$                 & $39.95$      & $923.58$ ($200$)                          &  $955.05$ ($200$)\\
					&FastAdv                      & $40.22$                 & $38.83$      & $963.77$ ($200$)                          &  $1017.66$ ($200$)\\
					&RotNetAdv                    & $40.67$                 & $39.28$      & $880.86$ ($200$)                          &  $910.55$ ($200$)\\
					\bottomrule
				\end{tabular}
		\end{small}
	\end{center}

\end{table*} 
 
 \begin{figure}[htp]
 	\centering
 	\includegraphics[width=0.5\textwidth]{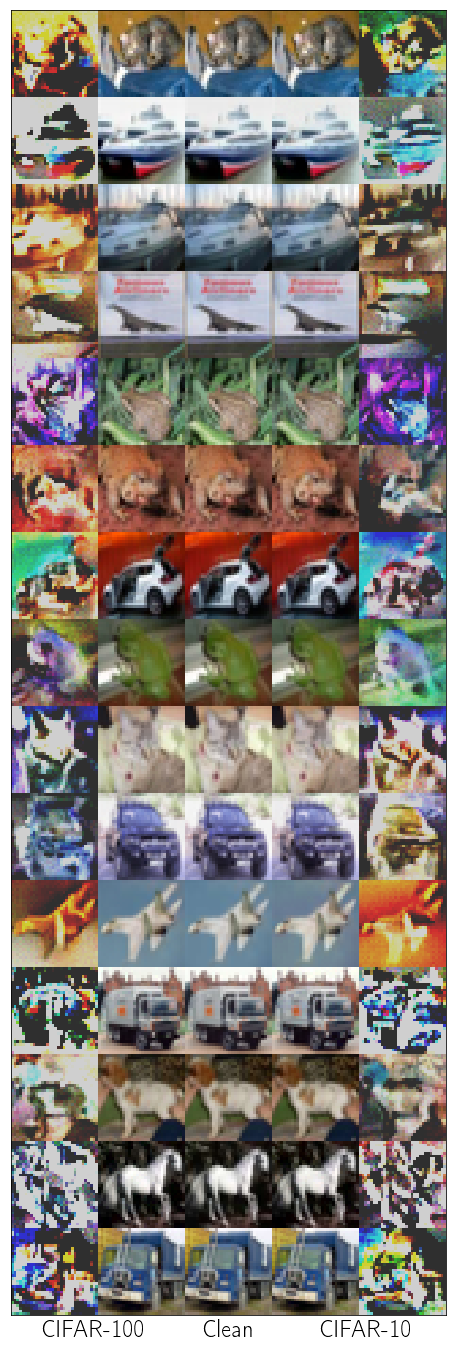}
 	\caption{Magnified difference and adversarial examples generated by AdvFlow alongside the clean data using flow-based models trained on CIFAR-100~\citep{krizhevsky2009learning} (left) and CIFAR-10~\citep{krizhevsky2009learning} (right).}
 	\label{fig:cifar100_samples}
 \end{figure}

\section{AdvFlow and Its Variations}\label{sec:ap:algs}

\subsection{AdvFlow}\label{sec:ap:advflow}

Algorithm~\ref{alg:advflow} summarizes the main adversarial attack approach introduced in this paper.
 
\begin{algorithm}[htp] 
	\caption{AdvFlow for inconspicuous black-box adversarial attacks\label{alg:advflow}} 
	\textbf{Input}: Clean data~${\mathbf{x}}$, true label~${y}$, pre-trained flow-based model~${\mathbf{f}(\cdot)}$.
	\\
	\textbf{Output}: Adversarial example~${\mathbf{x}'}.$
	\\
	\textbf{Parameters}: noise variance~${\sigma^2}$, learning rate~${\alpha}$, population size~${n_{p}}$, maximum number of queries~${Q}$.
	\begin{algorithmic}[1]
		\State Initialize ${\boldsymbol{\mu}}$ randomly.
		\State Compute ${\mathbf{z}_{clean} = \mathbf{f}^{-1}(\mathbf{x})}$.
		\For {$q=1,2,\ldots, \lfloor Q/n_p\rfloor$}
		\State Draw ${n_p}$ samples from ${\boldsymbol{\delta}_z}=\boldsymbol{\mu} + \sigma \boldsymbol{\epsilon}$ where $\boldsymbol{\epsilon}\sim\mathcal{N}(\boldsymbol{\epsilon}|\mathbf{0}, I)$.
		\State Set ${\mathbf{z}_k = \mathbf{z}_{clean} + {\boldsymbol{\delta}_z}_k}$ for all ${k=1,\ldots, n_p}$.
		\State Calculate ${\mathcal{L}_k = \mathcal{L}\big(\mathrm{proj}_{\mathcal{S}}\big(\mathbf{f}(\mathbf{z}_k)\big)\big)}$ for all ${k=1,\ldots, n_p}$.
		\State Normalize ${\hat{\mathcal{L}}_k = \big(\mathcal{L}_k - \mathrm{mean}(\boldsymbol{\mathcal{L}})\big)/\mathrm{std}(\boldsymbol{\mathcal{L}})}$.
		\State Compute ${\nabla_{\boldsymbol{\mu}}J(\boldsymbol{\mu}, \sigma) = \tfrac{1}{n_p}\sum_{k=1}^{n_p}\hat{\mathcal{L}}_k \boldsymbol{\epsilon}_{k}}$.
		\State Update ${\boldsymbol{\mu} \leftarrow \boldsymbol{\mu} - \alpha \nabla_{\boldsymbol{\mu}}J(\boldsymbol{\mu}, \sigma)}$.
		\EndFor
		\State Output ${\mathbf{x}'} = \mathrm{proj}_{\mathcal{S}}\big(\mathbf{f}(\mathbf{z}_{clean} + \boldsymbol{\mu})\big)$.
	\end{algorithmic} 
\end{algorithm}

\subsection{Greedy AdvFlow~\cite{dolatabadi2020greedy}}\label{sec:ap:greedy_advflow}

We can modify Algorithm~\ref{alg:advflow} so that it stops upon reaching a data point where it is adversarial.
To this end, we only have to actively check whether we have generated a data sample for which the C\&W cost of Eq.~\eqref{eq:CW_loss} is zero or not.
Besides, instead of using all of the generated samples to update the mean of the latent Gaussian ${\boldsymbol{\delta}_z}$, we can select the top-$K$ for which the C\&W loss is the lowest.
Then, we update the mean by taking the average of these latent space data points.
Applying these changes, we get a new algorithm coined \textit{Greedy AdvFlow}.
This approach is given in Algorithm~\ref{alg:advflow_greedy}.
 
\begin{algorithm}[htp] 
	\caption{Greedy AdvFlow for inconspicuous black-box adversarial attacks\label{alg:advflow_greedy}} 
	\textbf{Input}: Clean data~${\mathbf{x}}$, true label~${y}$, pre-trained flow-based model~${\mathbf{f}(\cdot)}$.
	\\
	\textbf{Output}: Adversarial example~${\mathbf{x}'}.$
	\\
	\textbf{Parameters}: noise variance~${\sigma^2}$, voting population~${K}$, population size~${n_{p}}$, maximum number of queries~${Q}$.
	\begin{algorithmic}[1]
		\State Initialize ${\boldsymbol{\mu}}$ randomly.
		\State Compute ${\mathbf{z}_{clean} = \mathbf{f}^{-1}(\mathbf{x})}$.
		\For {$q=1,2,\ldots, \lfloor Q/n_p\rfloor$}
		\State Draw ${n_p}$ samples from ${\boldsymbol{\delta}_z}=\boldsymbol{\mu} + \sigma \boldsymbol{\epsilon}$ where $\boldsymbol{\epsilon}\sim\mathcal{N}(\boldsymbol{\epsilon}|\mathbf{0}, I)$.
		\State Set ${\mathbf{z}_k = \mathbf{z}_{clean} + {\boldsymbol{\delta}_z}_k}$ for all ${k=1,\ldots, n_p}$.
		\State Calculate ${\mathcal{L}_k = \mathcal{L}\big(\mathrm{proj}_{\mathcal{S}}\big(\mathbf{f}(\mathbf{z}_k)\big)\big)}$ for all ${k=1,\ldots, n_p}$.
		\If {any $\mathcal{L}_k$ becomes $0$}:
			\State Output the ${\mathbf{x}'} = \mathrm{proj}_{\mathcal{S}}\big(\mathbf{f}(\mathbf{z}_{k})\big)$ for which $\mathcal{L}_k = 0$ as the adversarial example.
			\State break
		\EndIf
		\State Find the top-$K$ samples ${\mathbf{z}_k}$ with the lowest score ${\mathcal{L}_k}$.
		\State Update ${\boldsymbol{\mu} \leftarrow \tfrac{1}{K}\sum_{k \in \mathrm{top-}K}{\boldsymbol{\delta}_z}_k}$.
		\EndFor
	\end{algorithmic} 
\end{algorithm}

Table~\ref{tab:success_greedy} shows the performance of the proposed method with respect to AdvFlow.
As can be seen, this way we can improve the success rate and number of required queries.
\begin{table*}[htp]
	\caption{Attack success rate and average (median) of the number of queries to generate an adversarial example for CIFAR-10~\citep{krizhevsky2009learning}, and SVHN~\citep{netzer2011reading}.
		The architecture of the classifier in all of the cases is Wide-ResNet-32~\citep{zagoruyko2016wresnet}.
	Also, all attacks are with respect to $\ell_{\infty}$ norm with $\epsilon_{\max}=8/255$.
    The hyperparameters used for Greedy AdvFlow are the same as Table~\ref{tab:ap:AdvFlow_details}.
    In each iteration, we select the top-4 data samples to update the mean.}
	\label{tab:success_greedy}
	\begin{center}
		\begin{small}
				\begin{tabular}{cccccc}
					\toprule
					\parbox[t]{2mm}{\multirow{2}{*}{\rotatebox[origin=c]{90}{Data}}}&                        & \multicolumn{2}{c}{Success Rate(\%) \textuparrow}   & \multicolumn{2}{c}{Avg. (Med.) of Queries \textdownarrow}\\
					\cmidrule(lr){3-4}\cmidrule(lr){5-6}
					&Defense                      & Greedy AdvFlow          & AdvFlow      & Greedy AdvFlow                            & AdvFlow\\
					\midrule
					\parbox[t]{2mm}{\multirow{4}{*}{\rotatebox[origin=c]{90}{CIFAR-10}}}
					&Vanilla                      & $99.12$                 & $99.42$      & $991.98$ ($460$)                          &  $950.07$ ($400$)\\
					&FreeAdv                      & $41.06$                 & $41.21$      & $842.37$ ($180$)                          &  $923.58$ ($200$)\\
					&FastAdv                      & $40.06$                 & $40.22$      & $904.78$ ($200$)                          &  $963.77$ ($200$)\\
					&RotNetAdv                    & $40.50$                 & $40.67$      & $821.80$ ($180$)                          &  $880.86$ ($200$)\\
					\midrule
					\parbox[t]{2mm}{\multirow{4}{*}{\rotatebox[origin=c]{90}{SVHN}}}
					&Vanilla                      & $92.40$                 & $90.42$      & $1305.06$ ($540$)                         &  $1582.87$ ($800$)\\
					&FreeAdv                      & $52.57$                 & $50.63$      & $816.54$ ($200$)                          &  $1095.68$ ($200$)\\
					&FastAdv                      & $43.03$                 & $41.39$      & $781.06$ ($240$)                          &  $1046.45$ ($400$)\\
					&RotNetAdv                    & $45.43$                 & $44.37$      & $653.82$ ($160$)                          &  $923.59$ ($200$)\\
					\bottomrule
				\end{tabular}
		\end{small}
	\end{center}
\end{table*}

\subsection{AdvFlow for High-resolution Images}

Despite their ease-of-use in generating low-resolution images, high-resolution image generation with normalizing flows is computationally demanding.
This issue is even more pronounced in the case of images with high variabilities, such as the ImageNet~\cite{russakovsky2015imagenet} dataset, which may require a lot of invertible transformations to model them.
To cope with this problem, we propose an adjustment to our AdvFlow algorithm.
Instead of generating the image in the high-dimensional space, we first map it to a low-dimension space using bilinear interpolation.
Then, we perform the AdvFlow algorithm to generate the set of candidate examples.
Next, we compute the adversarial perturbations in the low-dimensional space and map them back to their high-dimensional representation using bilinear upsampling.
These perturbations are then added to the original target image, and the rest of the algorithm continues as before.
Figure~\ref{fig:ap:high_res} shows the block-diagram of the proposed solution for high-resolution data.
Moreover, the updated AdvFlow procedure is summarized in Algorithm~\ref{alg:advflow_high_res}.
Changes are highlighted in \textcolor{red}{red}.

\tikzstyle{function} = [rectangle, draw, fill=blue!20, text width=1.cm, text badly centered, node distance=1cm, minimum height=1.cm, inner sep=0pt]
\tikzstyle{line} = [draw, -latex']
\tikzstyle{downsampler} = [draw, trapezium, trapezium angle=67.5, fill=red!20, node distance=1.5cm, minimum height=2em, minimum width=9em, rotate=-90]
\tikzstyle{upsampler} = [draw, trapezium, trapezium angle=67.5, fill=red!20, node distance=1.5cm, minimum height=2em, minimum width=9em, rotate=90]
\tikzstyle{branch}=[fill,shape=circle,minimum size=3pt,inner sep=0pt]
\tikzstyle{block} = [circle, draw, fill=green!20, rounded corners, minimum height=0.05cm]

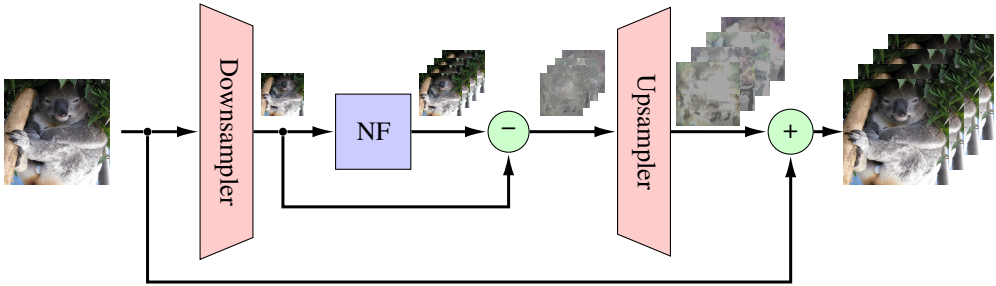
\begin{figure*}[b!]
	\centering
\begin{tikzpicture}[node distance = 1.25cm, auto]
\node [cloud] at (0,2) (Input) {\includegraphics[width=.1\textwidth]{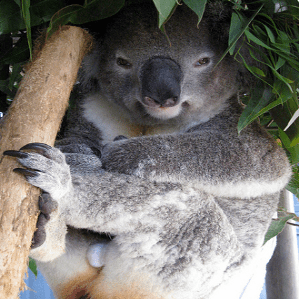}};
\node [branch] at (1.2,2) (br1){};
\node [cloud] at (1.2,0) (n1) {};
\node [downsampler] at (2.25, 2) (dw) {Downsampler};
\node [branch]  at (3,2) (br2){};
\node [function] at (4.2, 2) (alg) {NF};
\node [block] at (6, 2) (sub) {\textminus};
\node [upsampler] at (7.8, 2) (up) {\rotatebox{180}{~~Upsampler~~}};
\node [block] at (9.75, 2) (add) {+};
\node [cloud] at (10.65,2) (Output) {};

\node [cloud] at (3,2.5) (Input) {\includegraphics[width=.04\textwidth]{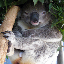}};

\node [cloud] at (5.4,2.8) (Input) {\includegraphics[width=.04\textwidth]{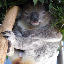}};
\node [cloud] at (5.3,2.7) (Input) {\includegraphics[width=.04\textwidth]{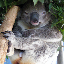}};
\node [cloud] at (5.2,2.6) (Input) {\includegraphics[width=.04\textwidth]{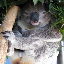}};
\node [cloud] at (5.1,2.5) (Input) {\includegraphics[width=.04\textwidth]{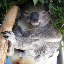}};

\node [cloud] at (7.0,2.8) (Input) {\includegraphics[width=.04\textwidth]{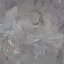}};
\node [cloud] at (6.9,2.7) (Input) {\includegraphics[width=.04\textwidth]{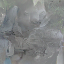}};
\node [cloud] at (6.8,2.6) (Input) {\includegraphics[width=.04\textwidth]{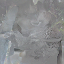}};
\node [cloud] at (6.7,2.5) (Input) {\includegraphics[width=.04\textwidth]{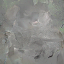}};

\node [cloud] at (9.25,3.1) (Input) {\includegraphics[width=.06\textwidth]{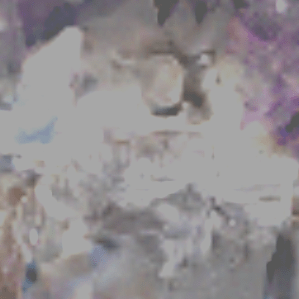}};
\node [cloud] at (9.05,2.9) (Input) {\includegraphics[width=.06\textwidth]{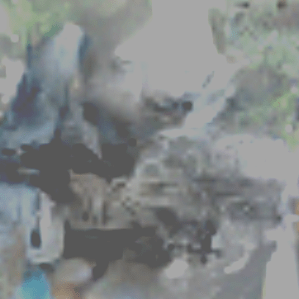}};
\node [cloud] at (8.85,2.7) (Input) {\includegraphics[width=.06\textwidth]{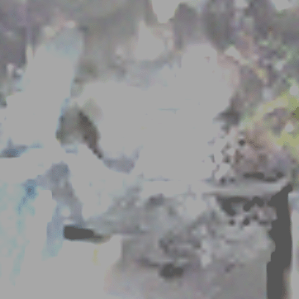}};
\node [cloud] at (8.65,2.5) (Input) {\includegraphics[width=.06\textwidth]{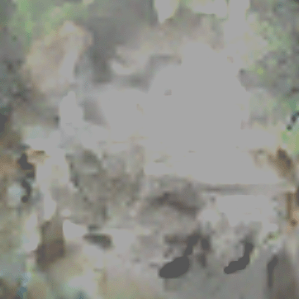}};

\node [cloud] at (11.75,2.6) (Input) {\includegraphics[width=.1\textwidth]{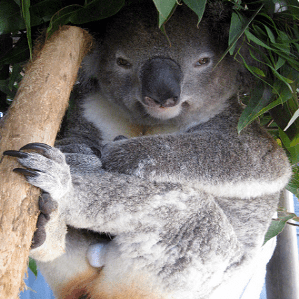}};
\node [cloud] at (11.55,2.4) (Input) {\includegraphics[width=.1\textwidth]{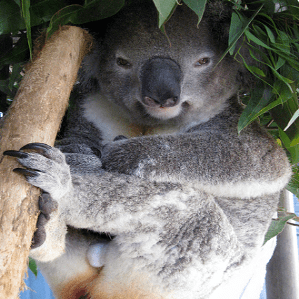}};
\node [cloud] at (11.35,2.2) (Input) {\includegraphics[width=.1\textwidth]{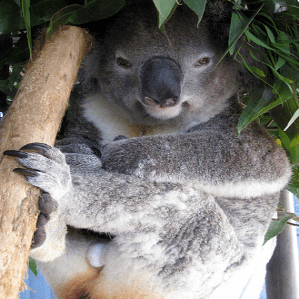}};
\node [cloud] at (11.15,2.) (Input) {\includegraphics[width=.1\textwidth]{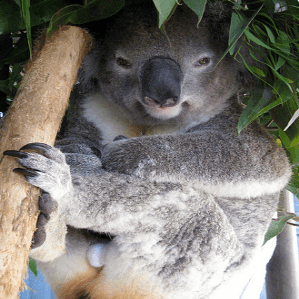}};

\draw[line width=0.4mm] (0.85,2) -- (br1);
\draw[-{Latex[length=3mm, width=1.5mm]}, line width=0.4mm] (br1) -- (dw);
\draw[line width=0.4mm] (dw) -- (br2);
\draw[-{Latex[length=3mm, width=1.5mm]}, line width=0.4mm] (br2) -- (alg);
\draw[-{Latex[length=3mm, width=1.5mm]}, line width=0.4mm] (alg) -- (sub);
\draw[-{Latex[length=3mm, width=1.5mm]}, line width=0.4mm] (sub) -- (up);
\draw[-{Latex[length=3mm, width=1.5mm]}, line width=0.4mm] (up) -- (add);
\draw[-{Latex[length=3mm, width=1.5mm]}, line width=0.4mm] (add) -- (Output);
\draw[-{Latex[length=3mm, width=1.5mm]}, line width=0.4mm] (br1) -- (1.2, 0) -| (add);
\draw[-{Latex[length=3mm, width=1.5mm]}, line width=0.4mm] (br2) -- (3., 1.) -| (sub);

\end{tikzpicture}
\caption{AdvFlow adjustment for high-resolution images. Instead of working with high-dimensional image, we downsample them. Then after generating candidate low-resolution perturbations, we map them to high-dimensions using a bilinear upsampler.}
\label{fig:ap:high_res}
\end{figure*}

\begin{algorithm}[htp] 
	\caption{AdvFlow for \textcolor{red}{high-resolution} black-box attack\label{alg:advflow_high_res}} 
	\textbf{Input}: Clean data~${\mathbf{x}}$, true label~${y}$, pre-trained flow-based model~${\mathbf{f}(\cdot)}$.
	\\
	\textbf{Output}: Adversarial example~${\mathbf{x}'}.$
	\\
	\textbf{Parameters}: noise variance~${\sigma^2}$, learning rate~${\alpha}$, population size~${n_{p}}$, maximum number of queries~${Q}$.
	\begin{algorithmic}[1]
		\State Initialize ${\boldsymbol{\mu}}$ randomly.
		\State \textcolor{red}{Downsample $\mathbf{x}$ and save it as $\mathbf{x}_{\mathrm{low}}$.}
		\State \textcolor{red}{Compute ${\mathbf{z}_{clean} = \mathbf{f}^{-1}(\mathbf{x}_{\mathrm{low}})}$.}
		\For {$q=1,2,\ldots, \lfloor Q/n_p\rfloor$}
		\State Draw ${n_p}$ samples from ${\boldsymbol{\delta}_z}=\boldsymbol{\mu} + \sigma \boldsymbol{\epsilon}$ where $\boldsymbol{\epsilon}\sim\mathcal{N}(\boldsymbol{\epsilon}|\mathbf{0}, I)$.
		\State Set ${\mathbf{z}_k = \mathbf{z}_{clean} + {\boldsymbol{\delta}_z}_k}$ for all ${k=1,\ldots, n_p}$.
		\State \textcolor{red}{Compute and upsample ${\boldsymbol{\gamma}_k =\mathbf{f}(\mathbf{z}_k) - \mathbf{x}_{\mathrm{low}}}$ for all ${k=1,\ldots, n_p}$.}
		\State \textcolor{red}{Calculate ${\mathcal{L}_k = \mathcal{L}\big(\mathrm{proj}_{\mathcal{S}}\big(\boldsymbol{\gamma}_k + \mathbf{x}\big)\big)}$ for all ${k=1,\ldots, n_p}$.}
		\State Normalize ${\hat{\mathcal{L}}_k = \big(\mathcal{L}_k - \mathrm{mean}(\boldsymbol{\mathcal{L}})\big)/\mathrm{std}(\boldsymbol{\mathcal{L}})}$.
		\State Compute ${\nabla_{\boldsymbol{\mu}}J(\boldsymbol{\mu}, \sigma) = \tfrac{1}{n_p}\sum_{k=1}^{n_p}\hat{\mathcal{L}}_k \boldsymbol{\epsilon}_{k}}$.
		\State Update ${\boldsymbol{\mu} \leftarrow \boldsymbol{\mu} - \alpha \nabla_{\boldsymbol{\mu}}J(\boldsymbol{\mu}, \sigma)}$.
		\EndFor
		\State \textcolor{red}{Upsample ${\boldsymbol{\gamma} = \mathbf{f}(\mathbf{z}_{clean} + \boldsymbol{\mu})} - \mathbf{x}_{\mathrm{low}}$.}
		\State \textcolor{red}{Output ${\mathbf{x}'} = \mathrm{proj}_{\mathcal{S}}\big(\boldsymbol{\gamma} + \mathbf{x}\big)$.}
	\end{algorithmic} 
\end{algorithm}

To test the performance of the proposed approach, we pre-train a flow-based model on a $64 \times 64$ version of ImageNet~\citep{russakovsky2015imagenet}.
The normalizing flow architecture and training hyperparameters are as shown in Table~\ref{tab:ap:nf_details}.
Furthermore, we use bandits with data-dependent priors~\citep{ilyas2019prior} and $\mathcal{N}$\textsc{Attack}~\cite{li2019nattack} to compare our model against them.
For bandits, we use the tuned hyperparameters for ImageNet~\citep{russakovsky2015imagenet} in the original paper~\citep{ilyas2019prior}.
Also, for $\mathcal{N}$\textsc{Attack}~\cite{li2019nattack} and AdvFlow we observe that the hyperparameters in Tables~\ref{tab:ap:NATTACK_details} and~\ref{tab:ap:AdvFlow_details} work best for the vanilla architectures in this dataset.
Thus, we kept them as before.

We use the nominated black-box methods to attack a classifier in less than $10,000$ queries.
We use pre-trained Inception-v3~\citep{szegedy2016rethinking}, ResNet-50~\citep{he2016deep}, and VGG-16~\citep{simonyan2015vgg} classifiers available on \texttt{torchvision} as our vanilla target models.
Also, a defended ResNet-50~\citep{he2016deep} model trained by \text{fast adversarial training}~\citep{wong2020fast} with FGSM ($\epsilon = 4/255$) is used for evaluation.
This model is available online on the official repository of fast adversarial training.\footnote{\href{https://github.com/anonymous-sushi-armadillo}{github.com/anonymous-sushi-armadillo}}   
  
Table~\ref{tab:success_imagenet} shows our experimental results on the ImageNet~\citep{russakovsky2015imagenet} dataset.
As can be seen, we get similar results to CIFAR-10~\citep{krizhevsky2009learning} and SVHN~\citep{netzer2011reading} experiments in Table~\ref{tab:success}: in case of vanilla architectures, we are performing slightly worse than $\mathcal{N}$\textsc{Attack}~\citep{li2019nattack}, while in the defended case we improve their performance considerably.
It should also be noted again that in all of the cases we are generating adversaries that look like the original data, and come from the same distribution.
This property is desirable in confronting adversarial example detectors.
Figure~\ref{fig:imagenet_samples} depicts a few adversarial examples generated by AdvFlow compared to $\mathcal{N}$\textsc{Attack}~\citep{li2019nattack} for a vanilla Inception-v3~\citep{szegedy2016rethinking} DNN classifier.
As seen, the AdvFlow perturbations tend to take the shape of the data to reduce the possibility of changing the underlying data distribution.
In contrast, $\mathcal{N}$\textsc{Attack} perturbations are pixel-level, independent additive noise that cause the adversarial example distribution to become different from that of the data.

\begin{table*}[htp]
	\caption{Attack success rate and average (median) of the number of queries needed to generate an adversarial example for ImageNet~\citep{russakovsky2015imagenet}.
		For a fair comparison, we first find the samples where all the attack methods are successful, and then compute the average (median) of queries for these samples.
		Note that for $\mathcal{N}$\textsc{Attack} and AdvFlow we check whether we arrived at an adversarial point every $200$ queries, hence the medians are multiples of $200$.
		All attacks are with respect to $\ell_{\infty}$ norm with $\epsilon_{\max}=8/255$.
	    * The accuracy is computed with respect to the $1000$ test data used for attack evaluation.}
	\label{tab:success_imagenet}
	\begin{center}
		\begin{small}
				\resizebox{\columnwidth}{!}{
				\begin{tabular}{ccccc}
					\toprule
					\parbox[t]{2mm}{\multirow{2}{*}{\rotatebox[origin=c]{90}{Data}}} &&Attack& \multicolumn{2}{c}{Bandits~\citep{ilyas2019prior} / $\mathcal{N}$\textsc{Attack}~\citep{li2019nattack} / SimBA~\citep{guo2019simba} / AdvFlow (ours)}\\
					\cmidrule(lr){4-5}
					&Arch.                               & Acc*(\%)          & \multicolumn{1}{c}{Success Rate(\%) \textuparrow}   & \multicolumn{1}{c}{Avg. (Med.) of Queries \textdownarrow}\\
					\midrule
					\parbox[t]{2mm}{\multirow{4}{*}{\rotatebox[origin=c]{90}{ImageNet}}}
					& Inception-v3                       & $99.20$           & $87.80$ / $\mathbf{95.06}$ / $80.95$ / $87.50$   & $1034.41$ ($430$) / $\mathbf{680.52}$ ($\mathbf{400}$) / $1481.12$ ($1142$) / $1516.34$ ($800$)\\
					& VGG16                              & $92.50$           & $95.46$ / $\mathbf{99.57}$ / $97.95$ / $97.51$   & $541.64$ ($\mathbf{166}$) / $\mathbf{395.61}$ ($200$) / $608.42$ ($486$) / $1239.03$ ($600$)\\
					& Van. ResNet                        & $95.00$           & $95.79$ / $\mathbf{99.47}$ / $98.42$ / $95.58$   & $948.90$ ($\mathbf{364}$) / $\mathbf{604.31}$ ($400$) / $701.92$ ($494$) / $1501.13$ ($800$)\\
					& Def. ResNet                        & $71.50$           & $50.77$ / $33.99$ / $47.55$ / $\mathbf{57.20}$   & $914.58$ ($404$) / $2170.82$ ($1200$) / $969.91$ ($696$) / $\mathbf{381.97}$ ($\mathbf{200}$)\\
					\bottomrule
				\end{tabular}}
		\end{small}
	\end{center}
\end{table*}

\begin{figure}[htp]
	\centering
	\begin{subfigure}{.19\textwidth}
		\centering
		\includegraphics[width=1.0\textwidth]{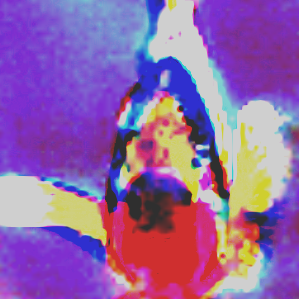}
	\end{subfigure}
    \hspace*{-0.425em}
	\begin{subfigure}{.19\textwidth}
		\centering
		\includegraphics[width=1.0\textwidth]{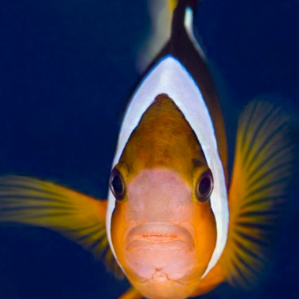}
	\end{subfigure}
	\hspace*{-0.425em}
	\begin{subfigure}{.19\textwidth}
		\centering
		\includegraphics[width=1.0\textwidth]{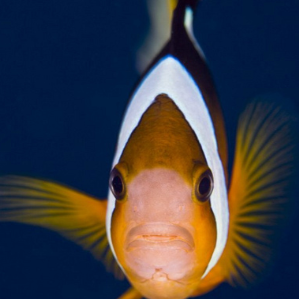}
	\end{subfigure}
	\hspace*{-0.425em}
	\begin{subfigure}{.19\textwidth}
		\centering
		\includegraphics[width=1.0\textwidth]{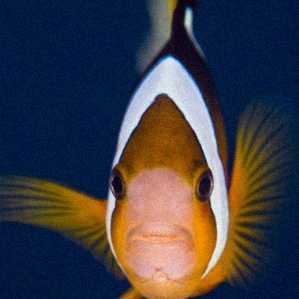}
	\end{subfigure}
	\hspace*{-0.425em}
	\begin{subfigure}{.19\textwidth}
		\centering
		\includegraphics[width=1.0\textwidth]{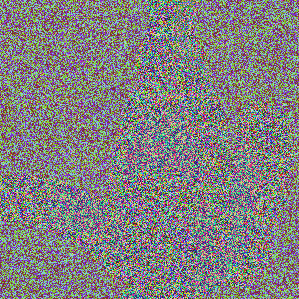}
	\end{subfigure}\\
	\begin{subfigure}{.19\textwidth}
		\centering
		\includegraphics[width=1.0\textwidth]{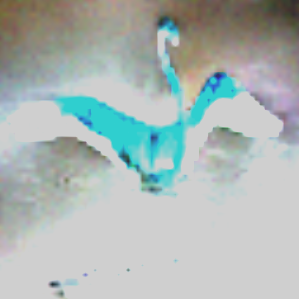}
	\end{subfigure}
	\hspace*{-0.425em}
	\begin{subfigure}{.19\textwidth}
		\centering
		\includegraphics[width=1.0\textwidth]{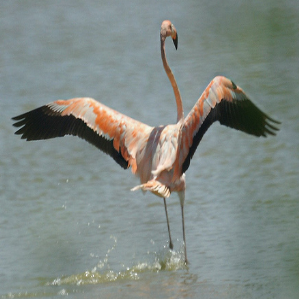}
	\end{subfigure}
	\hspace*{-0.425em}
	\begin{subfigure}{.19\textwidth}
		\centering
		\includegraphics[width=1.0\textwidth]{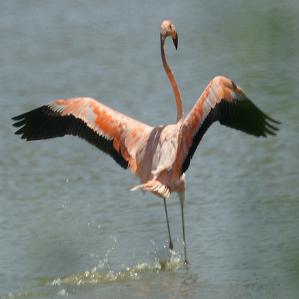}
	\end{subfigure}
	\hspace*{-0.425em}
	\begin{subfigure}{.19\textwidth}
		\centering
		\includegraphics[width=1.0\textwidth]{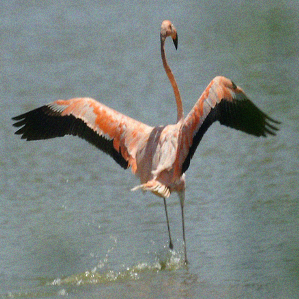}
	\end{subfigure}
	\hspace*{-0.425em}
	\begin{subfigure}{.19\textwidth}
		\centering
		\includegraphics[width=1.0\textwidth]{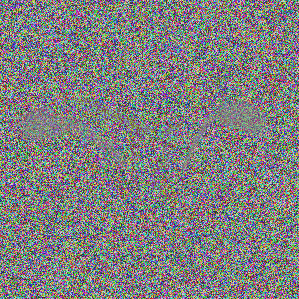}
	\end{subfigure}\\
	\begin{subfigure}{.19\textwidth}
		\centering
		\includegraphics[width=1.0\textwidth]{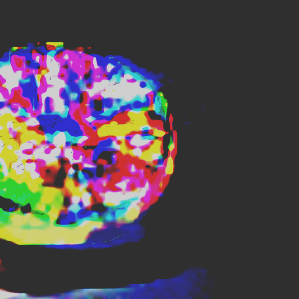}
	\end{subfigure}
	\hspace*{-0.425em}
	\begin{subfigure}{.19\textwidth}
		\centering
		\includegraphics[width=1.0\textwidth]{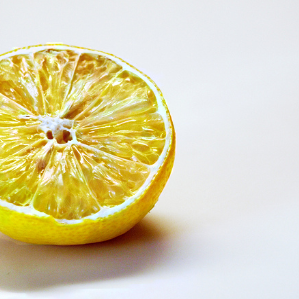}
	\end{subfigure}
	\hspace*{-0.425em}
	\begin{subfigure}{.19\textwidth}
		\centering
		\includegraphics[width=1.0\textwidth]{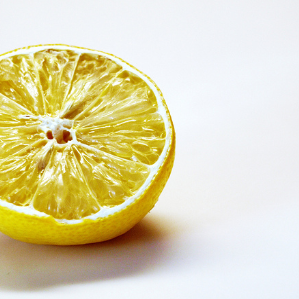}
	\end{subfigure}
	\hspace*{-0.425em}
	\begin{subfigure}{.19\textwidth}
		\centering
		\includegraphics[width=1.0\textwidth]{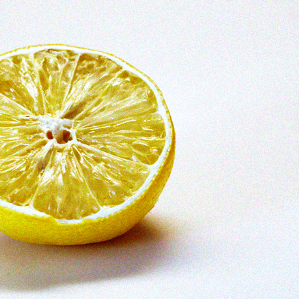}
	\end{subfigure}
	\hspace*{-0.425em}
	\begin{subfigure}{.19\textwidth}
		\centering
		\includegraphics[width=1.0\textwidth]{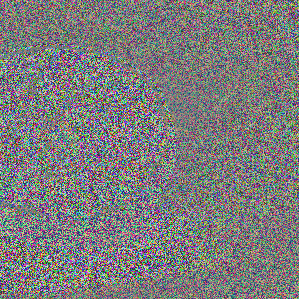}
	\end{subfigure}\\
	\begin{subfigure}{.19\textwidth}
		\centering
		\includegraphics[width=1.0\textwidth]{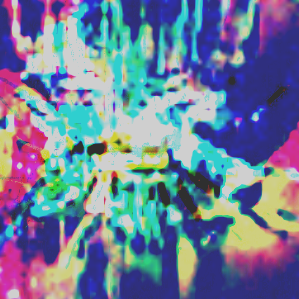}
	\end{subfigure}
	\hspace*{-0.425em}
	\begin{subfigure}{.19\textwidth}
		\centering
		\includegraphics[width=1.0\textwidth]{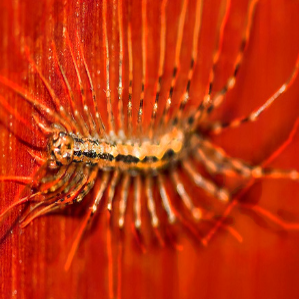}
	\end{subfigure}
	\hspace*{-0.425em}
	\begin{subfigure}{.19\textwidth}
		\centering
		\includegraphics[width=1.0\textwidth]{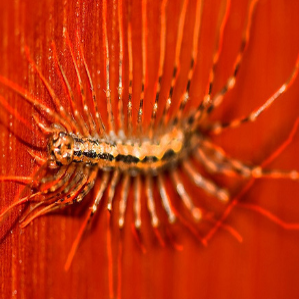}
	\end{subfigure}
	\hspace*{-0.425em}
	\begin{subfigure}{.19\textwidth}
		\centering
		\includegraphics[width=1.0\textwidth]{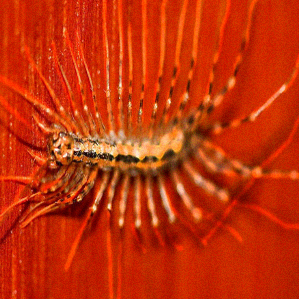}
	\end{subfigure}
	\hspace*{-0.425em}
	\begin{subfigure}{.19\textwidth}
		\centering
		\includegraphics[width=1.0\textwidth]{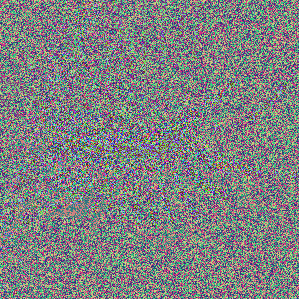}
	\end{subfigure}\\
	\begin{subfigure}{.19\textwidth}
		\centering
		\includegraphics[width=1.0\textwidth]{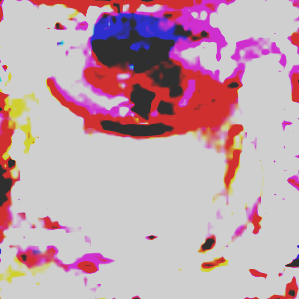}
		\caption*{}
	\end{subfigure}
	\hspace*{-0.425em}
	\begin{subfigure}{.19\textwidth}
		\centering
		\includegraphics[width=1.0\textwidth]{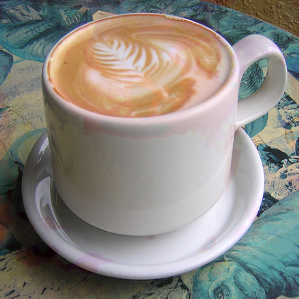}
		\caption*{AdvFlow (ours)}
	\end{subfigure}
	\hspace*{-0.425em}
	\begin{subfigure}{.19\textwidth}
		\centering
		\includegraphics[width=1.0\textwidth]{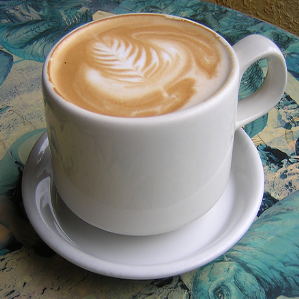}
		\caption*{Clean Image}
	\end{subfigure}
	\hspace*{-0.425em}
	\begin{subfigure}{.19\textwidth}
		\centering
		\includegraphics[width=1.0\textwidth]{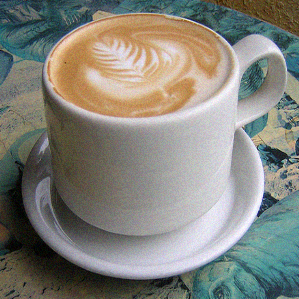}
		\caption*{$\mathcal{N}$\textsc{Attack}}
	\end{subfigure}
	\hspace*{-0.425em}
	\begin{subfigure}{.19\textwidth}
		\centering
		\includegraphics[width=1.0\textwidth]{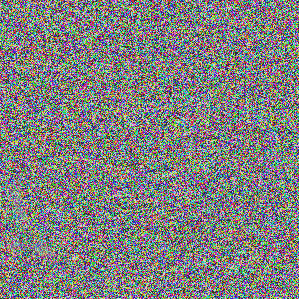}
		\caption*{}
	\end{subfigure}
	\caption{Magnified difference and adversarial examples generated by AdvFlow (ours) and $\mathcal{N}$\textsc{Attack}~\cite{li2019nattack} alongside the clean data for ImageNet~\citep{russakovsky2015imagenet} dataset. The target network architecture is Inception-v3~\citep{szegedy2016rethinking}.}
	\label{fig:imagenet_samples}
\end{figure}
\clearpage

\subsection{AdvFlow for People in a Hurry!}

Alternatively, one can use the plain structure of AdvFlows for black-box adversarial attacks.
To this end, we are only required to initialize the normalizing flow randomly.
This way, however, we will be getting random-like perturbations as in $\mathcal{N}$\textsc{Attack}~\cite{li2019nattack} since we have not trained the flow-based model.
Using this approach, we can surpass the performance of the baselines in vanilla DNNs.
In fact, giving away a little bit of performance is the price we pay to force the perturbation to have a data-like structure so that the adversaries have a similar distribution to the data.
Table~\ref{tab:success_imagenet_random} summarizes the performance of randomly initialized AdvFlow in contrast to bandits with data-dependent priors~\citep{ilyas2019prior} and $\mathcal{N}$\textsc{Attack}~\cite{li2019nattack} for vanilla ImageNet~\citep{russakovsky2015imagenet} classifiers.
Also, Figure~\ref{fig:imagenet_samples_random} shows a few adversarial examples generated by randomly initialized AdvFlows in this case. 

\begin{table*}[htp]
	\caption{Attack success rate and average (median) of the number of queries needed to generate an adversarial example for ImageNet~\citep{russakovsky2015imagenet}.
		For a fair comparison, we first find the samples where all the attack methods are successful, and then compute the average (median) of queries for these samples.
		Note that for $\mathcal{N}$\textsc{Attack} and AdvFlow we check whether we arrived at an adversarial point every $200$ queries, hence the medians are multiples of $200$.
		All attacks are with respect to $\ell_{\infty}$ norm with $\epsilon_{\max}=8/255$.
		* The accuracy is computed with respect to the $1000$ test data used for attack evaluation.}
	\label{tab:success_imagenet_random}
	\begin{center}
		\begin{small}
				\resizebox{\columnwidth}{!}{
					\begin{tabular}{cccc}
						\toprule
						&Attack& \multicolumn{2}{c}{Bandits~\citep{ilyas2019prior} / $\mathcal{N}$\textsc{Attack}~\citep{li2019nattack} / SimBA~\citep{guo2019simba} / AdvFlow (Random Init.)}\\
						\cmidrule(lr){3-4}
						Arch.                              & Acc*(\%)      & \multicolumn{1}{c}{Success Rate(\%) \textuparrow} & \multicolumn{1}{c}{Avg. (Med.) of Queries \textdownarrow}\\
						\midrule
						Inception-v3                       & $99.20$       & $87.80$ / $95.06$ / $80.95$ / $\mathbf{97.78}$    & $1081.13$ ($452$) / $745.10$ ($400$) / $1537.71$ ($1173$) / $\mathbf{375.00}$ ($\mathbf{200}$)\\
						VGG16                              & $92.50$       & $95.46$ / $\mathbf{99.57}$ / $97.95$ / $99.35$    & $586.14$ ($174$)  / $418.33$ ($\mathbf{200}$) / $637.96$ ($503$) / $\mathbf{299.89}$ ($\mathbf{200}$)\\
						Van. ResNet                        & $95.00$       & $95.79$ / $\mathbf{99.47}$ / $99.16$ / $98.42$    & $1053.44$ ($415$) / $672.53$ ($400$) / $758.32$ ($523$) / $\mathbf{370.63}$ ($\mathbf{200}$)\\
						\bottomrule
				\end{tabular}}
		\end{small}
	\end{center}
	\vskip -0.1in
\end{table*}
\begin{figure}[htp]
	\centering
	\begin{subfigure}{.19\textwidth}
		\centering
		\includegraphics[width=1.0\textwidth]{/Samples/AdvFlow/1_pert.png}
	\end{subfigure}
	\hspace*{-0.425em}
	\begin{subfigure}{.19\textwidth}
		\centering
		\includegraphics[width=1.0\textwidth]{/Samples/AdvFlow/1_adv.png}
	\end{subfigure}
	\hspace*{-0.425em}
	\begin{subfigure}{.19\textwidth}
		\centering
		\includegraphics[width=1.0\textwidth]{/Samples/AdvFlow/1_clean.png}
	\end{subfigure}
	\hspace*{-0.425em}
	\begin{subfigure}{.19\textwidth}
		\centering
		\includegraphics[width=1.0\textwidth]{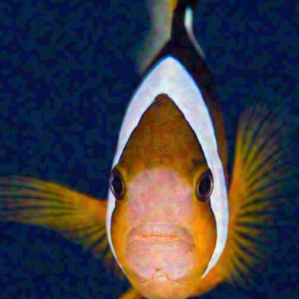}
	\end{subfigure}
	\hspace*{-0.425em}
	\begin{subfigure}{.19\textwidth}
		\centering
		\includegraphics[width=1.0\textwidth]{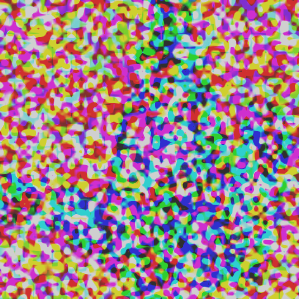}
	\end{subfigure}\\
	\begin{subfigure}{.19\textwidth}
		\centering
		\includegraphics[width=1.0\textwidth]{/Samples/AdvFlow/2_pert.png}
	\end{subfigure}
	\hspace*{-0.425em}
	\begin{subfigure}{.19\textwidth}
		\centering
		\includegraphics[width=1.0\textwidth]{/Samples/AdvFlow/2_adv.png}
	\end{subfigure}
	\hspace*{-0.425em}
	\begin{subfigure}{.19\textwidth}
		\centering
		\includegraphics[width=1.0\textwidth]{/Samples/AdvFlow/2_clean.png}
	\end{subfigure}
	\hspace*{-0.425em}
	\begin{subfigure}{.19\textwidth}
		\centering
		\includegraphics[width=1.0\textwidth]{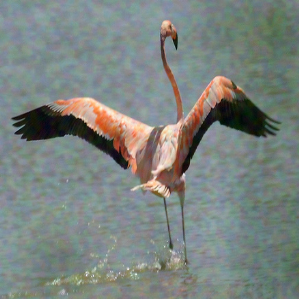}
	\end{subfigure}
	\hspace*{-0.425em}
	\begin{subfigure}{.19\textwidth}
		\centering
		\includegraphics[width=1.0\textwidth]{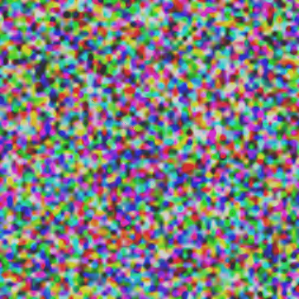}
	\end{subfigure}\\
	\begin{subfigure}{.19\textwidth}
		\centering
		\includegraphics[width=1.0\textwidth]{/Samples/AdvFlow/4_pert.png}
	\end{subfigure}
	\hspace*{-0.425em}
	\begin{subfigure}{.19\textwidth}
		\centering
		\includegraphics[width=1.0\textwidth]{/Samples/AdvFlow/4_adv.png}
	\end{subfigure}
	\hspace*{-0.425em}
	\begin{subfigure}{.19\textwidth}
		\centering
		\includegraphics[width=1.0\textwidth]{/Samples/AdvFlow/4_clean.png}
	\end{subfigure}
	\hspace*{-0.425em}
	\begin{subfigure}{.19\textwidth}
		\centering
		\includegraphics[width=1.0\textwidth]{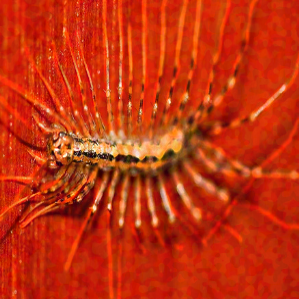}
	\end{subfigure}
	\hspace*{-0.425em}
	\begin{subfigure}{.19\textwidth}
		\centering
		\includegraphics[width=1.0\textwidth]{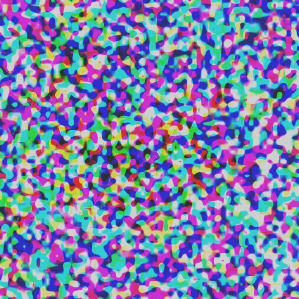}
	\end{subfigure}\\
	\begin{subfigure}{.19\textwidth}
		\centering
		\includegraphics[width=1.0\textwidth]{/Samples/AdvFlow/6_pert.png}
		\caption*{}
	\end{subfigure}
	\hspace*{-0.425em}
	\begin{subfigure}{.19\textwidth}
		\centering
		\includegraphics[width=1.0\textwidth]{/Samples/AdvFlow/6_adv.png}
		\caption*{Pre-trained AdvFlow}
	\end{subfigure}
	\hspace*{-0.425em}
	\begin{subfigure}{.19\textwidth}
		\centering
		\includegraphics[width=1.0\textwidth]{/Samples/AdvFlow/6_clean.png}
		\caption*{Clean Image}
	\end{subfigure}
	\hspace*{-0.425em}
	\begin{subfigure}{.19\textwidth}
		\centering
		\includegraphics[width=1.0\textwidth]{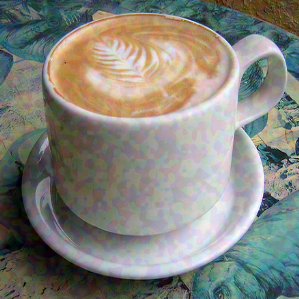}
		\caption*{Random AdvFlow}
	\end{subfigure}
	\hspace*{-0.425em}
	\begin{subfigure}{.19\textwidth}
		\centering
		\includegraphics[width=1.0\textwidth]{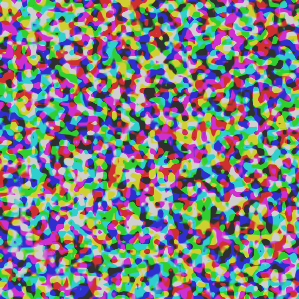}
		\caption*{}
	\end{subfigure}
	\caption{Magnified difference and adversarial examples generated by AdvFlow in trained and un-trained scenarios for ImageNet~\citep{russakovsky2015imagenet} dataset. The target network architecture is Inception-v3~\citep{szegedy2016rethinking}.}
	\label{fig:imagenet_samples_random}
\end{figure}
\end{document}